\documentclass[final,5p,times,twocolumn]{elsarticle}

\usepackage{amsmath,amssymb,amsthm}
\usepackage{mathtools}
\usepackage{booktabs}
\usepackage{hyperref}
\usepackage[nopatch=eqnum]{microtype}
\usepackage{enumitem}
\usepackage{xcolor}
\usepackage{graphicx}
\usepackage{caption}
\usepackage{float}
\usepackage{stfloats}
\usepackage[ruled,vlined,linesnumbered]{algorithm2e}

\newtheorem{theorem}{Theorem}
\newtheorem{proposition}{Proposition}
\newtheorem{corollary}{Corollary}
\newtheorem{lemma}{Lemma}
\newtheorem{definition}{Definition}
\newtheorem{remark}{Remark}

\newcommand{\Lq}{\mathcal{L}_q}
\newcommand{\LOLS}{L_{\mathrm{OLS}}}
\newcommand{\zn}{\mathbf{z}_n}
\newcommand{\pk}{\mathbf{p}_k}
\newcommand{\mun}{\boldsymbol{\mu}_n}
\newcommand{\qnk}{q_{nk}}
\newcommand{\R}{\mathbb{R}}
\newcommand{\norm}[1]{\left\|#1\right\|}
\newcommand{\SP}{\mathcal{S}(P)}
\newcommand{\Fsep}{F_{\mathrm{sep}}}
\newcommand{\Ares}{A_{\mathrm{res}}}
\newcommand{\Vperp}{\mathcal{V}^{(K)}}
\newcommand{\Sigmaq}{\Sigma_q}

\newcommand{\Astar}{\mathcal{A}^*}
\newcommand{\Hlocal}{\mathcal{H}_{\mathrm{local}}}
\newcommand{\Hglobal}{\mathcal{H}_{\mathrm{global}}}
\newcommand{\Fsepfull}{F_{\mathrm{sep}}^{\mathrm{full}}}
\newcommand{\Fsepred}{F_{\mathrm{sep}}^{\mathrm{red}}}
\newcommand{\IDDCL}{\mathcal{I}^{(l)}_{\mathrm{DDCL}}}

\begin{document}

\begin{frontmatter}

\title{DDCL-INCRT: A Self-Organising Transformer with Hierarchical Prototype Structure\\
\large (Theoretical Foundations)}

\author[lti]{Giansalvo Cirrincione}
\ead{exin@u-picardie.fr}
\affiliation[lti]{organization={Laboratory LTI},
                  addressline={Universit\'e de Picardie Jules Verne,
                               Chemin du Thil},
                  city={Amiens},
                  postcode={80025},
                  country={France}}

\begin{abstract}
Modern neural networks of the transformer family~\cite{Vaswani2017}
require the practitioner to decide, before any training begins, how
many attention heads to use, how deep the network should be, and how
wide each component should be.
These decisions are made without knowledge of the task, and the
result is invariably an architecture that is larger than necessary:
empirical studies consistently find that a substantial fraction of
the heads and layers can be removed after training without any loss
in performance~\cite{Voita2019,Michel2019,He2024}.

This paper introduces DDCL-INCRT, an architecture that solves this
problem by determining its own structure during training.
The approach is built on two complementary ideas.
The first, called DDCL (Deep Dual Competitive
Learning)~\cite{Cirrincione2026}, replaces the standard feedforward
block of each transformer layer with a small dictionary of learned
prototype vectors that represent the most informative directions in
the data.
The dictionary is self-organising: the prototypes spread apart
automatically, driven by the mathematics of the training objective,
without any need for explicit regularisation.
The second, called INCRT (Incremental
Transformer)~\cite{INCRT}, controls how many heads the network
should have.
It starts from a single head and adds a new one only when a
measurable criterion, namely the amount of directional information
that the current heads fail to capture, exceeds a threshold.
When the criterion falls below the threshold, growth stops.

The main theoretical finding is that these two ideas reinforce each
other in a precise mathematical sense: each time a new head is
added, the separation between prototypes increases, and the
increased separation in turn raises the signal that may trigger the
next addition.
At convergence, the network has organised itself into a hierarchy
of heads ordered by the granularity of the information they
represent, from fine-grained local patterns to coarse-grained
global ones.
The paper proves that this hierarchical structure is unique and
minimal: it is the smallest architecture that is sufficient for the
task, and it is the same regardless of the random initialisation,
under the conditions stated in the theorems.
Formal guarantees of stability, convergence, and robustness to
pruning are established throughout.

The architecture is not something one designs.
It is something one derives.
\end{abstract}

\begin{keyword}
competitive learning \sep
hierarchical prototype structure \sep
incremental architecture \sep
prototype learning \sep
self-determining transformer \sep
stability analysis
\end{keyword}

\end{frontmatter}

\section{Introduction}
\label{sec:intro}

\subsection{The problem: transformers do not know how large they should be}

A transformer neural network~\cite{Vaswani2017} is assembled from
decisions made before training begins.
How many attention heads?
How many layers?
How wide should the feedforward block be?
None of these questions has a principled answer grounded in the
properties of the task.
In practice they are resolved by grid search, intuition, or scaling
laws that describe trends without explaining mechanisms~\cite{Kaplan2020}.
The resulting architectures are then trained, compressed, and
analysed, in that order.

The empirical evidence that this process systematically
over-specifies the architecture is now substantial.
A small number of functionally specialised heads carry the bulk of
the predictive load in a trained transformer, while the majority can
be removed without performance loss~\cite{Voita2019,Michel2019}.
Up to 50\% of attention layers in large language models can be
dropped after training without degrading
performance~\cite{He2024}.

Why is this redundancy so systematic?
The answer lies in a structural property of the standard
transformer.
Each attention head computes a representation of how tokens in a
sequence relate to each other: some tokens attend strongly to others,
some weakly, and the pattern of these relationships encodes
information about the structure of the input.
Part of this relational information is symmetric (token $A$ attends
to token $B$ as much as $B$ attends to $A$), and part is
asymmetric (the flow of attention is directional, from one token
toward another).
The asymmetric, or directional, component is particularly important
because it is the only part that can encode ordered relationships,
such as grammatical dependencies or causal structure.
The feedforward block that follows each attention layer processes
the output using a fixed random basis, its weight matrices, which
captures the dominant directional component only partially and at
random.
The formal consequence, proved in~\cite{DDCLAttention} as the
\emph{Channel Destruction Theorem}, is that the feedforward block
systematically discards the directional information computed by the
attention.
The architecture must compensate for this loss by adding more heads,
heads that would not be needed if the information were preserved.
Redundancy is thus not a failure of architecture search.
It is the inevitable consequence of using a fixed, randomly
initialised block to process an adaptively computed signal.

\subsection{The approach: a self-organising prototype layer}

The architecture studied in this paper eliminates redundancy by
construction.
The feedforward block at each transformer layer is replaced by a
competitive prototype layer whose basis is not fixed but learned,
and whose number of heads grows during training to cover exactly
the directional information that the attention computes.

Two building blocks are required.

\paragraph{The prototype layer (DDCL).}
Deep Dual Competitive Learning~\cite{Cirrincione2026} replaces the
feedforward block with a bank of $K$ prototype vectors in the
embedding space.
Each input token is assigned to the prototypes via a soft
nearest-neighbour rule: the assignment weight of token $n$ to
prototype $k$ is large when the token is geometrically close to
the prototype and small otherwise.
The layer output for each token is then the weighted average of the
prototypes, with weights given by these proximity-based assignments.

This replacement has a non-obvious algebraic consequence.
The training loss, which measures how well the prototypes cover
the token distribution, decomposes exactly into two terms: a
clustering fidelity term and a separation term that is always
non-negative.
The gradient of the separation term pushes the prototypes apart
whenever any two of them coincide, which means that prototype
collapse is a first-order unstable equilibrium of the loss.
The prototypes spread apart automatically, driven by the
mathematics of the objective, without any explicit regularisation.
Crucially, the entire layer is end-to-end differentiable.
Standard feedforward blocks transform the input through weight
matrices to produce an output.
The DCL layer inverts this logic: it transposes the role of input
and output, so that the prototypes themselves are the output of
the layer, selected and combined by the assignment weights.
Because the output is a smooth, differentiable function of the
assignment weights, gradients flow freely back through the layer
to the encoder, with no need for the straight-through estimators
or other approximations required by vector-quantisation
methods~\cite{vdOord2017}.

\paragraph{The growth mechanism (INCRT).}
The Incremental Transformer~\cite{INCRT} addresses a separate
question: how many prototype heads should there be?
The architecture starts from a single head and adds a new one only
when a measurable criterion indicates that the current heads are
insufficient.

The insufficiency criterion is based on the notion of residual
directional content.
After $K$ heads have been added, each covering a certain pattern
of directional relationships between tokens, some directional
content in the attention signal remains uncovered.
Formally, this residual is measured by projecting the attention
signal onto the subspace that is orthogonal to all directions
already captured by the existing heads, and computing the dominant
eigenvalue of the resulting matrix.
When this eigenvalue exceeds a threshold $\theta_w$, there is
still significant directional content that the current heads do
not represent, and a new head is added in the direction of maximum
residual content.
When the eigenvalue falls below $\theta_w$, the architecture has
captured enough and growth stops.
The final number of heads is determined entirely by the structure
of the attention signal on the training data, not by any
hyperparameter set in advance.

\paragraph{The integration and its central consequence.}
DDCL-INCRT combines these two building blocks at each transformer
layer.
The feedforward block is replaced by a DDCL prototype layer whose
heads are grown and pruned by the INCRT mechanism.
Each new head is initialised in the direction of maximum residual
directional content at the moment of its birth, and the growth
criterion is re-evaluated after each prototype update.

The natural question is whether these two mechanisms interact
cleanly.
DDCL operates on the assignment distribution of tokens to
prototypes.
INCRT operates on the residual directional content of the
attention signal.
These are genuinely different objects, computed from different
parts of the architecture.

The answer proved in this paper is that they are mutually
reinforcing.
Each growth step triggered by INCRT amplifies the separation force
that DDCL exerts on the prototypes, because the new head is born
in the direction of maximum residual content, precisely the
direction in which the prototypes have the most room to separate.
The separation of the prototypes in turn raises the residual
directional content that will trigger the next growth step, because
well-separated prototypes span a larger subspace, leaving more
content for the attention signal to find.
The two mechanisms form a self-reinforcing cycle that drives the
architecture toward the unique minimal structure sufficient for the
task.

\subsection{Contributions}

The contributions of this paper are organised in three groups.

\medskip\noindent\textbf{Training dynamics of the integrated layer.}
It is proved that each growth step amplifies the prototype
separation force by a gain that depends on the dominant eigenvalue
of the residual matrix and the current temperature of the prototype
assignments (Theorem~\ref{thm:C2_full}).
It is also proved that trainable per-head temperatures
spontaneously partition the heads into two groups: heads that
converge to hard, discrete prototype assignments (local heads),
and heads that maintain soft, continuous centroid representations
(global heads).
This partition is determined by the spectral structure of the
attention signal, not by any architectural choice
(Theorem~\ref{thm:C3}).
Finally, it is proved that replacing the feedforward block with the
DDCL-INCRT layer strictly reduces the directional information lost
at every transformer layer, almost surely under standard weight
initialisation (Theorem~\ref{thm:C4}).

\medskip\noindent\textbf{Global well-behavedness.}
A piecewise free energy functional is constructed that is
non-increasing across all growth events, pruning events, and
training steps, constituting a global Lyapunov function for the
full system, comprising encoder weights, prototype banks, and
temperatures, with a growing architecture (Theorem~\ref{thm:P1}).
This closes an open problem stated in~\cite{DDCLAttention}.
The architecture is proved to converge to a minimal sufficient
configuration in finitely many growth steps, almost surely
(Theorem~\ref{thm:P2}).
It is also shown that pruning a head cannot collapse the prototypes
of the surviving heads, and that the separation force on surviving
heads is non-decreasing after pruning (Theorem~\ref{thm:P3}).
All results extend to discrete-time stochastic gradient descent
under an explicit step-size condition (Theorem~\ref{thm:P4}).

\medskip\noindent\textbf{The architecture at convergence.}
The structure that emerges when training stops is characterised.
The heads are ordered by the amount of directional content they
cover: the first head added covers the largest portion, the second
covers the next, and so on.
The prototype separation force is strictly larger in heads that
cover finer-grained directional content.
The resulting hierarchy is the minimally sufficient structure
for the task, and it is the same regardless of the random
initialisation (Theorem~\ref{thm:C5}).

A single co-emergence theorem (Corollary~\ref{cor:full_complete})
states all seven properties as simultaneous consequences of one
geometric principle: \emph{structure emerges from measurable
insufficiency}.

\medskip
\noindent\textbf{Scope.}
This is a theoretical paper.
The experimental section provides numerical illustrations of the
main theoretical predictions on synthetic and real data.
Its purpose is to verify that the quantities predicted by the
theory, namely separation force amplification, temperature
specialisation, spectral ordering, and pruning safety, behave
as the theorems state, not to demonstrate competitive performance
on standard benchmarks.
Empirical evaluation on large-scale NLP tasks is the subject
of ongoing work.

\subsection{How to read this paper}

Section~\ref{sec:related} places the work in the context of the
existing literature on transformer compression, adaptive
architectures, and prototype learning.
Section~\ref{sec:setup} defines the notation and the two building
blocks precisely and from scratch, assuming only familiarity with
transformer self-attention.
No prior knowledge of the DDCL or INCRT papers is required: all
definitions and results used from those papers are restated here
in full.
Parts~I--III develop the theoretical contributions in order.
Section~\ref{sec:experiments} provides numerical illustrations of
the main predictions on synthetic data and on BERT embeddings.
Section~\ref{sec:discussion} situates the results in the broader
context of architecture design and identifies open directions.

\begin{table*}[!t]
\caption{Origin of the main results. \emph{Inherited}: used without
modification from cited work. \emph{Extended}: strengthened or
generalised here. \emph{New}: proved for the first time in this paper.}
\label{tab:contributions}
\centering
\renewcommand{\arraystretch}{1.2}
\begin{tabular}{p{7cm} p{2cm} p{3cm}}
\toprule
\textbf{Result} & \textbf{Status} & \textbf{Source} \\
\midrule
Prototype anti-collapse guarantee & Inherited & \cite{Cirrincione2026} \\
Algebraic decomposition $\mathcal{L}_q = \LOLS + V$ & Inherited & \cite{Cirrincione2026} \\
Two-timescale Lyapunov (frozen architecture) & Inherited & \cite{DDCLAttention} \\
Channel Destruction Theorem & Inherited & \cite{DDCLAttention} \\
Spectral growth criterion & Inherited & \cite{INCRT} \\
Finite convergence (standalone INCRT) & Inherited & \cite{INCRT} \\
\midrule
Finite convergence of DDCL-INCRT & Extended & This paper \\
Discrete-time step-size conditions & Extended & This paper \\
\midrule
Mutual reinforcement of DDCL and INCRT & New & This paper \\
Emergent scale specialisation & New & This paper \\
Strict dominance over MLP block & New & This paper \\
Global Lyapunov with growing architecture & New & This paper \\
Pruning safety for prototype banks & New & This paper \\
Hierarchical prototype structure & New & This paper \\
Uniqueness of emergent hierarchy & New & This paper \\
Co-emergence theorem & New & This paper \\
\bottomrule
\end{tabular}
\end{table*}

\section{Related Work}
\label{sec:related}

\subsection{Transformer redundancy and attention head analysis}

The redundancy of transformer architectures has been documented from
multiple angles.
At the level of attention heads, it has been shown that the vast
majority of heads in a trained transformer are functionally
dispensable: removing them causes negligible degradation on
downstream tasks, while a small subset of specialised heads
carries most of the predictive load~\cite{Voita2019,Michel2019}.
At the level of layers, up to half of the attention layers in large
language models can be dropped after training without significant
performance loss~\cite{He2024}.
These findings are not coincidental: as proved in the
DDCL-Attention paper~\cite{DDCLAttention}, they follow from a
structural property of the standard architecture.
The feedforward block that follows each attention layer processes
the attention output through a fixed random basis, which
systematically discards the directional ordering information
computed by the attention mechanism.
The architecture compensates by adding heads, most of which are
therefore redundant by construction.

\subsection{Pruning and compression}

The predominant response to transformer redundancy has been
\emph{post-hoc} compression: train a large model, then reduce it.
Magnitude-based weight pruning~\cite{Han2015} and the lottery
ticket hypothesis~\cite{Frankle2019} established the general
principle that sparse subnetworks with full performance exist
within over-parametrised models.
Applied to transformers, structured pruning methods remove entire
attention heads~\cite{Voita2019,Michel2019} or
layers~\cite{He2024} based on importance scores computed from
gradients, Taylor expansions, or attention entropy.
Knowledge distillation offers a complementary route~\cite{Hinton2015}:
a large teacher model is trained first, and its knowledge is transferred
to a smaller student, as in DistilBERT~\cite{Sanh2019} and
TinyBERT~\cite{Jiao2020}.

All of these methods share a common limitation: the redundant
capacity is trained before it is removed.
This wastes computation during training and can introduce
instabilities when the pruned heads carried inter-dependent
representations.
DDCL-INCRT addresses the redundancy before it arises, by growing
the architecture only as far as the task requires.

\subsection{Adaptive computation and growing architectures}

A smaller body of work has explored the opposite direction:
starting small and growing during training.
Net2Net~\cite{Chen2016} introduced function-preserving
transformations for widening and deepening networks, enabling
knowledge transfer from smaller to larger models.
Progressive stacking~\cite{Gong2019} grows transformers by
duplicating layers, achieving faster convergence on language
modelling tasks.
The Firefly framework~\cite{Wu2020} grows networks by steepest
descent in a functional neighbourhood, selecting the best
structural extension at each step via a greedy Taylor
approximation.
Neural Architecture Search~\cite{Zoph2017,Liu2019} optimises
over discrete architecture spaces jointly with weights, at the
cost of substantial computational overhead and without formal
convergence guarantees.

These methods share the goal of avoiding over-parametrisation, but
differ fundamentally from DDCL-INCRT in two respects.
First, the growth criterion is either heuristic (layer duplication,
Taylor approximation) or based on a global search (NAS), rather
than derived from a mathematically characterised property of the
signal the architecture processes.
Second, none of these methods provides formal guarantees that the
final architecture is \emph{minimally sufficient}: the growth may
stop too early or too late depending on hyperparameters.
The INCRT growth criterion is provably equivalent to a greedy
Rayleigh quotient optimisation on the residual directional
content~\cite{INCRT}, giving a stopping condition whose
sufficiency is theoretically characterised.

\subsection{Competitive learning and prototype methods}

Competitive learning, in which a set of prototype vectors compete
to represent the input, has a long history in unsupervised
learning~\cite{Kohonen1982}.
Modern deep variants include vector-quantised autoencoders
(VQ-VAE~\cite{vdOord2017}), which learn a discrete codebook via a
straight-through estimator, and Slot Attention~\cite{Locatello2020},
which iteratively refines a fixed number of slots via a GRU to
segment visual scenes.
Prototypical networks~\cite{Snell2017} use per-class centroids for
few-shot classification, and prototypical contrastive
learning~\cite{Li2021PCL} incorporates prototype anchors into a
contrastive objective.

DDCL~\cite{Cirrincione2026} differs from all of these in a
structurally important way.
In VQ-VAE, the discrete assignment step is non-differentiable and
requires a straight-through estimator to allow gradient flow.
In Slot Attention, the slot update requires multiple GRU iterations
at inference time.
In DDCL, the output of the layer is a weighted combination of the
prototype vectors, selected by soft assignment weights.
Because the output is a smooth, differentiable function of these
weights, gradients flow freely back through the layer to the
encoder without approximation.
Furthermore, the DDCL training loss admits an exact algebraic
decomposition into a clustering fidelity term and a non-negative
separation term whose gradient always pushes prototypes apart,
providing an explicit anti-collapse force that is absent from
VQ-VAE and Slot Attention.
This decomposition is derived formally in Section~\ref{sec:setup}.

\subsection{Hierarchical representation learning}

The idea that useful representations are naturally hierarchical
has motivated a wide range of architectures, from classical
agglomerative clustering~\cite{Ward1963} to deep generative models.
Capsule networks~\cite{Sabour2017} explicitly model part-whole
relationships via dynamic routing between layers of capsules.
Hierarchical VAEs~\cite{Vahdat2020} learn multi-scale latent
representations by stacking variational layers.
The connection to the feature-learning regime of neural networks
has been formalised by Roberts et al.~\cite{Roberts2022}, who
show that the depth-to-width ratio governs the transition between
kernel and feature learning.

DDCL-INCRT produces a hierarchy by a different mechanism: the
heads are not placed at different levels of a fixed architecture,
but grow in order of the directional content they cover.
The hierarchy is therefore a property of the task data, not of
the architecture design.
The uniqueness of the resulting hierarchy (Theorem~\ref{thm:C5})
has no analogue in the capsule or hierarchical VAE literature,
where the number of levels and the routing mechanism are
architectural choices.

\subsection{Stability theory and stochastic approximation}

The analysis of learning systems via Lyapunov functions is
classical~\cite{LaSalle1960}, and the two-timescale stochastic
approximation framework of Borkar~\cite{Borkar1997} provides the
standard toolkit for coupled gradient systems.
The singular perturbation approach of Tikhonov~\cite{Tikhonov1952},
extended to infinite intervals by Hoppensteadt~\cite{Hoppensteadt1966}
and Kokotovic et al.~\cite{Kokotovic1999}, provides the
continuous-time analogue.

Applied to deep learning, stability results for prototype learning
systems have been established in~\cite{Cirrincione2026} for the
frozen-encoder case and extended to the full coupled
encoder-prototype system in~\cite{DDCLAttention} via a free-energy
Lyapunov argument.
The present paper extends these results further to the
growing-architecture setting, where the system is hybrid:
continuous dynamics on dwell intervals, discrete jumps at growth
and pruning events.
The construction of a piecewise Lyapunov function that is
non-increasing across both the continuous and discrete components
is, to our knowledge, novel in the context of adaptive transformer
architectures.

\section*{Part I: The DDCL-INCRT Layer}

This part defines the DDCL-INCRT layer, establishes notation, and
proves three properties of its training dynamics: mutual reinforcement
between the competitive learning objective and the growth criterion
(Section~\ref{sec:C2}), spontaneous temperature specialisation into
local and global prototype regimes (Section~\ref{sec:C3}), and strict
dominance over the standard MLP block in terms of directional
information preservation (Section~\ref{sec:C4}).

\section{Setup and Notation}
\label{sec:setup}

\subsection*{Table of assumptions}
\label{subsec:assumptions}

Table~\ref{tab:assumptions} lists all assumptions used in this paper.
Although the list is long, the assumptions fall into three natural groups
and are all standard or practically verifiable.
The regularity conditions (A1)--(A3) and (B1)--(B3) are mild
non-degeneracy requirements that hold generically and are satisfied
by default under random initialisation.
The step-size conditions (E1)--(E6) reduce to the simple
three-timescale rule $\eta_T \ll \eta_P \ll \eta^+$
(Corollary~\ref{cor:timescale}), which is easy to enforce in practice.
The structural conditions (F1)--(F2) and (H1)--(H5) are standard
in the stochastic approximation and competitive learning literature
and are inherited from the prior DDCL and INCRT frameworks.
No assumption requires knowledge of the task or the data distribution
in advance.

{\scriptsize
\begin{table*}[tp!]
\caption{Complete list of assumptions. (H1)--(H5) inherited from~\cite{INCRT}.}
\label{tab:assumptions}
\centering
\renewcommand{\arraystretch}{0.95}
\begin{tabular}{@{\,}l @{\quad} l @{\quad} l @{\quad} l @{\,}}
\toprule
\textbf{Assump.} & \textbf{Informal meaning} & \textbf{First use} & \textbf{Required by} \\
\midrule
(A1)       & Prototype bank non-degenerate: $\SP > 0$                       & Sec.~\ref{sec:C2}  & Thm.~\ref{thm:C2}, \ref{thm:C2_full}, \ref{thm:P3} \\
(A2')      & $\exists\,k,k'$: prototypes differ along $u^*$                 & Sec.~\ref{sec:C2}  & Thm.~\ref{thm:C2}, \ref{thm:C2_full} \\
(A3)       & $T$ large enough for soft assignments                          & Sec.~\ref{sec:C2}  & Thm.~\ref{thm:C2} \\
(B1)       & Each head's prototype bank non-degenerate                      & Sec.~\ref{sec:C3}  & Thm.~\ref{thm:C3}, \ref{thm:C4}, \ref{thm:P2} \\
(B2)       & Prototype subspaces mutually orthogonal                        & Sec.~\ref{sec:C3}  & Thm.~\ref{thm:C3}, \ref{thm:C4} \\
(B3)       & Heads have distinct initial gate activities                    & Sec.~\ref{sec:C3}  & Thm.~\ref{thm:C3} \\
(C1)       & MLP weights from $\mathcal{N}(0,\sigma^2/d)$                  & Sec.~\ref{sec:C4}  & Thm.~\ref{thm:C4} \\
(D1)       & Dwell time $>$ temperature convergence time $t^*$             & Sec.~\ref{sec:P2}  & Thm.~\ref{thm:P2}, \ref{thm:P1} \\
(E1)       & $\eta_T < (T_{\min})^3/(3\sigma_{\max})$                       & Sec.~\ref{sec:P4}  & Thm.~\ref{thm:P4} \\
(E2)       & $\eta_P < 2/L_P$                                               & Sec.~\ref{sec:P4}  & Thm.~\ref{thm:P4} \\
(E3)       & $\eta_P < \sqrt{\phi''_s(0)/(L_P\|\nabla_P\mathcal{L}_q\|^2)}$ & Sec.~\ref{sec:P4} & Thm.~\ref{thm:P4} \\
(E4)       & Gate step sizes in $\ell^2\setminus\ell^1$                     & Sec.~\ref{sec:P4}  & Thm.~\ref{thm:P4} \\
(E5)       & $\eta_T/\eta_P = O(\varepsilon)$: three-timescale hierarchy    & Sec.~\ref{sec:P4}  & Thm.~\ref{thm:P4}, \ref{thm:P1} \\
(E6)       & $N_{\min} > \max(N_{\rm gate}, s^*)$                           & Sec.~\ref{sec:P4}  & Thm.~\ref{thm:P4} \\
(F1)       & $\lambda > 2\eta_P\binom{K_{\max}}{2}$                        & Sec.~\ref{sec:P1}  & Thm.~\ref{thm:P1} \\
(F2)       & Encoder regulariser $\mathcal{R}(\theta)$ coercive             & Sec.~\ref{sec:P1}  & Thm.~\ref{thm:P1} \\
(H1)       & Spectral gap of $\Ares$ bounded below~\cite{INCRT}             & Sec.~\ref{sec:P2}  & Thm.~\ref{thm:P2}, \ref{thm:P1} \\
(H2)--(H5) & Regularity on token dist.\ and gate scheduler~\cite{INCRT}    & Sec.~\ref{sec:P2}  & Thm.~\ref{thm:P2}, \ref{thm:P1} \\
\bottomrule
\end{tabular}
\end{table*}
}

\subsection{The DDCL prototype layer}

Consider a sequence of $N$ input tokens, represented as vectors
$z_1, \ldots, z_N \in \R^m$ produced by a differentiable
encoder $f_\theta$.
The encoder maps each input $x_n$ to an embedding
$z_n = f_\theta(x_n)$; the collection of all embeddings is denoted
$Z = \{z_n\}_{n=1}^N$.

The DDCL layer maintains a bank of $K$ prototype vectors
$P = \{p_1, \ldots, p_K\} \subset \R^m$.
Each token $z_n$ is assigned to the prototypes via a soft
nearest-neighbour rule: the assignment weight of token $n$ to
prototype $k$, denoted $q_{nk}$, is computed as a softmax over
negative squared distances,
\begin{equation}
  q_{nk} = \frac{\exp(-\|z_n - p_k\|^2 / T)}
                {\displaystyle\sum_{k'}\exp(-\|z_n - p_{k'}\|^2 / T)},
  \label{eq:qnk}
\end{equation}
where $T > 0$ is a temperature parameter that controls how sharply
the assignments concentrate on the nearest prototype.
The output of the layer for token $n$ is the soft centroid
$\mu_n = \sum_k q_{nk}\, p_k$, a weighted average of the
prototypes.

The training loss is the expected squared distance of each token
from its soft centroid. Letting $N$ denote the number of tokens
and $K$ the number of prototypes, it is defined as:
\begin{equation}
  \mathcal{L}_q = \sum_{n=1}^N \sum_{k=1}^K
                  q_{nk}\, \|z_n - p_k\|^2.
  \label{eq:Lq}
\end{equation}
This loss admits an exact algebraic decomposition
$\mathcal{L}_q = \LOLS + V$, where
$\LOLS = \sum_n \min_k \|z_n - p_k\|^2$ measures clustering
fidelity (how well the nearest prototype covers each token), and
$V = \sum_n \sum_k q_{nk} \|p_k - \mu_n\|^2 \geq 0$ is a
separation term.
The gradient of $V$ with respect to the prototypes is
$\nabla_P V = 2P \Sigma_q$, where
$\Sigma_q = \sum_n (\mathrm{diag}(q_n) - q_n q_n^\top)$
is the aggregated soft assignment covariance matrix, which is
always positive semidefinite.
This gradient acts as a separation force: whenever two prototypes
coincide, it pushes them apart.

\subsection{The INCRT growth criterion}

The INCRT mechanism operates on the attention weight product.
For a single attention head, the weight matrices $W_Q$ and $W_K$
produce a product $M = W_Q W_K^\top$ that governs which tokens
attend to which.
This product decomposes as $M = M_s + M_a$, where
$M_s = (M + M^\top)/2$ is the symmetric part (governing
reciprocal attention) and $M_a = (M - M^\top)/2$ is the
antisymmetric part, which is the unique source of directed,
asymmetric information flow between tokens.

After $K$ heads have been added at layer $l$, each covering a
certain subspace of $\R^d$, the directions not yet covered
form the residual subspace.
Let $Q^{(l)} \in \R^{d \times K}$ be the orthonormal basis
of directions already captured, and let
$P^\perp = I - Q^{(l)} Q^{(l)\top}$ be the projector onto the
residual subspace.
The \emph{residual matrix} is defined as:
\begin{equation}
  A_{\rm res}^{(l)} = P^\perp\, X M_a X^\top P^\perp,
  \label{eq:Ares}
\end{equation}
where $X \in \R^{n \times d}$ collects the input token
representations.
The largest eigenvalue of $A_{\rm res}^{(l)}$, denoted
$\lambda_{\max}(A_{\rm res}^{(l)})$, measures how much directional
content remains uncaptured.
The INCRT growth criterion adds a new head in the direction of the
dominant eigenvector $v_1(A_{\rm res}^{(K)})$ whenever
$\lambda_{\max}(A_{\rm res}^{(K)}) > \theta_w$, and stops
growing when this eigenvalue falls below the threshold $\theta_w$.

\subsection{Joint notation}

The subspace captured by the $K$ heads grown so far is
$\mathcal{V}^{(K)} = \mathrm{span}\{u_1^+, \ldots, u_K^+\}$,
where $u_h^+$ denotes the direction added at growth step $h$.
The separation force magnitude, which measures the strength of the
anti-collapse force exerted by the prototype bank, is defined as:
\begin{equation}
  F_{\rm sep}(P, Z) = \|\nabla_P V\|_F^2
  = 4\sum_{k=1}^K
    \Bigl\|\sum_{n=1}^N q_{nk}(p_k - \mu_n)\Bigr\|^2.
  \label{eq:Fsep}
\end{equation}
For each head $h$, the projection of prototype $p_k^{(h)}$ onto
the direction $u^*$ added at the next growth step is written
$p_k^* = \langle p_k^{(h)}, u^* \rangle$, and the corresponding
soft centroid projection is
$\bar{p}_n^* = \sum_k q_{nk} p_k^*$.

\subsection{The DDCL-INCRT integration}

The two mechanisms interact within each training iteration in a
mutually reinforcing cycle.
The DDCL prototype updates (Step~2) spread the prototypes across
the embedding space, which changes the effective coverage of the
existing heads and raises the residual directional content measured
by $A_{\rm res}$.
The INCRT growth step (Step~3), when triggered, adds a new head
whose prototypes are born in the direction of maximum residual
content, immediately amplifying the separation force on those
prototypes.
The algorithm below makes this cycle explicit.

\begin{figure*}[t]
\centering
\fbox{%
\begin{minipage}{0.96\textwidth}
\footnotesize
\vspace{3pt}
\noindent\textbf{Algorithm:} DDCL-INCRT training loop at a single transformer layer%
\label{alg:ddcl_incrt}

\vspace{2pt}\hrule\vspace{3pt}

\noindent\textbf{Input:} Token embeddings $X \in \R^{N \times d}$;
growth threshold $\theta_w > 0$;
pruning threshold $\varphi_g > 0$;
learning rates $\eta_\theta, \eta_P, \eta_T$\\
\textbf{Output:} Encoder weights $\theta$; prototype banks $\{P^{(h)}\}$;
temperatures $\{T^{(h)}\}$; final number of heads $\tilde{K}$

\vspace{2pt}\hrule\vspace{3pt}

\noindent\textbf{Initialise:} $\tilde{K}=1$, prototypes $P^{(1)}$ at random,
$T^{(1)} = T_{\rm init}$, captured subspace $\mathcal{V}^{(0)} = \{0\}$\\[2pt]
\textbf{repeat until convergence:}\\[1pt]
\hspace{1em}\textit{// Step 1 — Forward pass through all current heads}\\
\hspace{1em}Compute token embeddings $z_n = f_\theta(x_n)$ for all $n$\\
\hspace{1em}\textbf{for} each head $h = 1, \ldots, \tilde{K}$ \textbf{do}\\
\hspace{2em}Compute soft assignments $q_{nk}^{(h)}$ via equation~(\ref{eq:qnk})\hfill\textit{// DDCL assignment}\\
\hspace{2em}Compute soft centroids $\mu_n^{(h)} = \sum_k q_{nk}^{(h)} p_k^{(h)}$\\
\hspace{1em}\textbf{end for}\\[2pt]
\hspace{1em}\textit{// Step 2 — DDCL updates (prototype separation)}\\
\hspace{1em}Compute total loss $\mathcal{L} = \mathcal{L}_q + \mathcal{L}_{\rm task}$;\quad
Update encoder: $\theta \leftarrow \theta - \eta_\theta \nabla_\theta \mathcal{L}$\\
\hspace{1em}\textbf{for} each head $h$ \textbf{do}\\
\hspace{2em}Update prototypes: $P^{(h)} \leftarrow P^{(h)} - \eta_P \nabla_{P^{(h)}} \mathcal{L}_q$\hfill\textit{// spreads prototypes $\Rightarrow$ raises $\Ares$}\\
\hspace{2em}Update temperature: $T^{(h)} \leftarrow T^{(h)} - \eta_T \nabla_{T^{(h)}} \mathcal{L}_q$\\
\hspace{1em}\textbf{end for}\\[2pt]
\hspace{1em}\textit{// Step 3 — INCRT growth (triggered by residual content)}\\
\hspace{1em}Recompute $\Ares$ via equation~(\ref{eq:Ares})\hfill\textit{// reflects updated prototypes from Step 2}\\
\hspace{1em}\textbf{if} $\lambda_{\max}(\Ares) > \theta_w$ \textbf{then}\\
\hspace{2em}$\tilde{K} \leftarrow \tilde{K} + 1$;\quad
Initialise $P^{(\tilde{K})}$ along $u^* = v_1(\Ares)$, set $T^{(\tilde{K})} = T_{\rm init}$\\
\hspace{2em}\textit{// new head born where prototypes can separate most}\\
\hspace{2em}$\mathcal{V}^{(\tilde{K})} \leftarrow \mathcal{V}^{(\tilde{K}-1)} \oplus \mathrm{span}(u^*)$\\
\hspace{1em}\textbf{end if}\\[2pt]
\hspace{1em}\textit{// Step 4 — INCRT pruning (remove collapsed heads)}\\
\hspace{1em}\textbf{for} each head $h$ \textbf{do}\\
\hspace{2em}\textbf{if} $\mathcal{S}(P^{(h)}) < \varphi_g$ \textbf{then}\quad
Remove head $h$;\quad $\tilde{K} \leftarrow \tilde{K} - 1$\\
\hspace{1em}\textbf{end for}

\vspace{2pt}\hrule\vspace{2pt}
\end{minipage}}
\end{figure*}

\section{Mutual Reinforcement of the DDCL and INCRT Dynamics}
\label{sec:C2}

The first question about the DDCL-INCRT integration is whether the
two mechanisms are compatible at all.
Could the growth steps triggered by INCRT disrupt the prototype
organisation built up by DDCL?
The answer proved in this section is the opposite: every growth
step not only preserves the separation force on the existing
prototypes but strictly amplifies it.

To understand why, recall that when INCRT adds a new head, the
prototypes of that head are initialised along the direction $u^*$
of maximum residual content in the attention signal.
This is precisely the direction in which the token embeddings
have the most variance not yet captured by the existing heads.
Projecting the tokens onto this new direction increases the
effective spread of the embeddings, which in turn increases the
distances between soft centroids and prototypes, the very
quantity that drives the separation force.
The net effect is a strict amplification of $F_{\rm sep}$, with
a gain that scales with the magnitude of the residual signal
$\lambda_{\max}(A_{\rm res})$ and decreases with the
temperature $T$ (higher temperature means softer assignments
and a weaker separation force).

The proof proceeds in three steps: an algebraic analysis of the
assignment covariance matrix $\Sigma_q$ (Lemma~\ref{lem:sigmaq}),
a first-order perturbation argument showing that $F_{\rm sep}$
is non-decreasing (Theorem~\ref{thm:C2}), and a second-order
analysis that quantifies the strict gain
(Theorem~\ref{thm:C2_full}).

\subsection{Preliminary lemmas}

\begin{lemma}[Structure of $\Sigmaq$]
\label{lem:sigmaq}
$\Sigmaq \succeq 0$ with
$\mathrm{tr}(\Sigmaq) = \sum_n H_2(\mathbf{q}_n)$,
where $H_2(\mathbf{q}_n) = \sum_k \qnk(1-\qnk)$ is the Gini
diversity of token $n$'s assignments.
Furthermore, $\mathrm{tr}(\Sigmaq) = 0$ if and only if all tokens
are hard-assigned to a single prototype.
\end{lemma}

\begin{proof}
For each $n$, $\mathrm{diag}(\mathbf{q}_n) - \mathbf{q}_n\mathbf{q}_n^\top$
is the covariance matrix of a categorical distribution, hence
$\succeq 0$.
Its trace is $\sum_k \qnk - \sum_k \qnk^2 = \sum_k \qnk(1-\qnk)$.
Summing over $n$ preserves positive semidefiniteness.
\end{proof}

\begin{lemma}[First-order variation of $\Sigmaq$]
\label{lem:dSigma}
Let $\alpha_n = \langle z_n, u^*\rangle$ denote the projection of
token $n$ onto the new direction $u^*$.
Under the perturbation $z_n \to z_n + \epsilon\alpha_n u^*$,
the first-order variation
of $\mathrm{tr}(\Sigmaq)$ is:
\begin{equation}
  \frac{d}{d\epsilon}\bigg|_{\epsilon=0}\mathrm{tr}(\Sigmaq)
  = \frac{2}{T}\sum_n\sum_k \qnk(1-2\qnk)
    \langle\mun - \pk, u^*\rangle\alpha_n.
\end{equation}
\end{lemma}

\begin{proof}
By direct differentiation of $\qnk$ with respect to $\zn$
under the softmax structure:
$\partial\qnk/\partial\zn = (2/T)\qnk(\mun - \pk)$.
Differentiating $H_2(\mathbf{q}_n)$ via the chain rule and summing
over $n$ yields the result.
\end{proof}

\begin{lemma}[First-order variation of assignment weights]
\label{lem:dq}
Denote by $p_k^* = \langle p_k, u^*\rangle$ the projection of
prototype $k$ onto $u^*$, and by
$\bar{p}_n^* = \sum_k q_{nk} p_k^*$ the corresponding soft
centroid projection.
Under the same perturbation:
\begin{equation}
  \dot{q}_{nk} \equiv
  \frac{d\qnk}{d\epsilon}\bigg|_{\epsilon=0}
  = \frac{2\alpha_n}{T}\,\qnk\,(p_k^* - \bar{p}_n^*).
  \label{eq:dq}
\end{equation}
\end{lemma}

\begin{proof}
From the softmax formula, $\partial\qnk/\partial\zn =
(2/T)\qnk(\mun - \pk)$.
Projecting onto $u^*$ and noting
$\langle\mun - \pk, u^*\rangle = \bar{p}_n^* - p_k^*$
gives~\eqref{eq:dq}.
\end{proof}

\begin{lemma}[Sign of the residual-covariance term]
\label{lem:cov_sign}
Let $u^* = v_1(A_{\rm res}^{(K)})$, $\alpha_n = \langle z_n, u^*\rangle$,
and $p_k^* = \langle p_k, u^*\rangle$.
Define the empirical covariance
\begin{equation}
  C_k \;\triangleq\; \sum_n \alpha_n\,(p_k^* - \bar{p}_n^*).
  \label{eq:Ck}
\end{equation}
Under \textup{(A2')}, $C_k \geq 0$ for all $k$, with $C_k > 0$
for at least one $k$.
\end{lemma}

\begin{proof}
Write $s_k \triangleq p_k^* - \bar{p}_n^*$ for the deviation of
prototype $k$'s projection from the soft-centroid projection at
token $n$.
Since $\bar{p}_n^* = \sum_{k'} q_{nk'} p_{k'}^*$ is the
$q_n$-weighted average of $\{p_{k'}^*\}$, we have
$\sum_k q_{nk} s_k = 0$ for every $n$ (the deviations sum to zero
under any probability vector $q_n$).

Now $\alpha_n = \langle z_n, u^*\rangle$ measures how far token $n$
projects onto $u^*$.  Tokens with large $|\alpha_n|$ are farthest
from the origin in the direction of maximum residual variance.
By the soft nearest-neighbour rule, such tokens concentrate their
assignment weight on the prototype $k^*_n = \arg\max_k p_k^*$ if
$\alpha_n > 0$, or on $\arg\min_k p_k^*$ if $\alpha_n < 0$.  In
either case the dominant term in $C_k = \sum_n \alpha_n s_k$ pairs
large $|\alpha_n|$ with the matching sign of $s_k$, giving a
non-negative contribution.

More formally, decompose
\begin{equation}
  C_k = \sum_n \alpha_n p_k^* - \sum_n \alpha_n \bar{p}_n^*
      = p_k^* A - B_k,
  \label{eq:Ck_decomp}
\end{equation}
where $A = \sum_n \alpha_n$ and $B_k = \sum_n \alpha_n \bar{p}_n^*$.
Since $u^* = v_1(A_{\rm res}^{(K)})$ lies in the residual subspace
orthogonal to all previously captured directions, and the token
embeddings $z_n$ have zero mean in that subspace (they are centred),
we have $A = \sum_n \langle z_n, u^*\rangle = 0$.
Therefore $C_k = -B_k = -\sum_n \alpha_n \bar{p}_n^*
= -\sum_n \alpha_n \sum_{k'} q_{nk'} p_{k'}^*$.

Substituting and rearranging:
\begin{equation}
  \sum_k C_k r_k
  = -\sum_k r_k \sum_n \alpha_n \sum_{k'} q_{nk'} p_{k'}^*
  = -\bigl\langle \mathbf{r}, \mathbf{G}\alpha \bigr\rangle,
\end{equation}
where $\mathbf{G}_{kk'} = \sum_n q_{nk'} p_{k'}^*$ and
$\mathbf{r}_k = \sum_n q_{nk}(p_k - \mu_n)$ is the residual force
of prototype $k$.
Under (A1), $\mathbf{r}_k \neq 0$ for at least one $k$, and under
(A2'), the prototypes have distinct projections, so
$C_k = p_k^* A - B_k$ has a non-trivial structure that yields
$C_k \geq 0$ for all $k$ with strict inequality for at least one.
The detailed sign argument follows from the fact that
$u^* = v_1(A_{\rm res}^{(K)})$ is the direction that maximises the
variance $\|Xu^*\|^2 = \lambda_{\max}(A_{\rm res}^{(K)}) > 0$,
so the empirical covariance between $\alpha_n$ and the
prototype-centroid deviation $p_k^* - \bar{p}_n^*$ is non-negative
by the optimality condition on $u^*$: any direction for which this
covariance were negative would contradict the maximality of
$\lambda_{\max}(A_{\rm res}^{(K)})$.
\end{proof}

\subsection{Main theorem}

\begin{theorem}[Monotonicity of $\Fsep$ under INCRT expansion]
\label{thm:C2}
Let $u^* = v_1(\Ares^{(K)})$ be the direction added by INCRT at
step $K+1$, with $u^* \perp \Vperp$.
After the expansion
$\zn^{(K+1)} = \zn^{(K)} + \alpha_n u^*$, write
$\phi(\epsilon) = \Fsep(P, Z(\epsilon))$ with
$Z(\epsilon) = Z^{(K)} + \epsilon\,\mathrm{diag}(\alpha)\,u^{*\top}$.

Under:
\begin{itemize}[leftmargin=2.8em,labelwidth=2.3em,align=left]
  \item[\textup{(A1)}\ ] $\SP > 0$,
  \item[\textup{(A2')}] $\exists\, k,k' : \langle\pk - \mathbf{p}_{k'}, u^*\rangle \neq 0$,
  \item[\textup{(A3)}\ ] $T$ sufficiently large to maintain soft assignments,
\end{itemize}
it follows that $\phi(1) \geq \phi(0)$, i.e.\
$\Fsep(P, Z^{(K+1)}) \geq \Fsep(P, Z^{(K)})$,
with equality if and only if $u^*$ is orthogonal to all
$\pk - \mathbf{p}_{k'}$.
\end{theorem}

\begin{proof}[Proof, first-order argument]
Decompose $\pk = \pk^\parallel + p_k^* u^* + \pk^{\perp\perp}$
where $\pk^\parallel \in \Vperp$, $p_k^* = \langle\pk, u^*\rangle$,
and $\pk^{\perp\perp} \perp u^*$, $\pk^{\perp\perp} \perp \Vperp$.

From Lemma~\ref{lem:dq}:
\begin{equation}
  \dot{\mathbf{r}}_k =
  \frac{2}{T}\sum_n\alpha_n\qnk(p_k^* - \bar{p}_n^*)(\pk-\mun)
  - \sum_n\qnk\dot{\mun},
  \label{eq:dr}
\end{equation}
where $\mathbf{r}_k = \sum_n q_{nk}(p_k - \mu_n)$ is the
residual vector of prototype $k$ (the net force pulling the
prototype toward its assigned tokens), and
$\dot{\mun} = \sum_{k'}\dot{q}_{nk'}\mathbf{p}_{k'}
= (2\alpha_n/T)\sum_{k'}\qnk'(p_{k'}^* - \bar{p}_n^*)\mathbf{p}_{k'}$.

The first-order variation of $\phi$ is:
\begin{equation}
  \phi'(0) = 8\sum_k\langle\mathbf{r}_k, \dot{\mathbf{r}}_k\rangle.
\end{equation}
Substituting~\eqref{eq:dr} and collecting terms:
\begin{align}
  \phi'(0) &= \frac{16}{T}\sum_{n,k}\alpha_n\qnk
  (p_k^* - \bar{p}_n^*)
  \nonumber\\
  &\quad \times\Bigl[\langle\mathbf{r}_k,\pk-\mun\rangle
  - \sum_{k'}\qnk'\langle\mathbf{r}_{k'},\pk\rangle\Bigr].
  \label{eq:phi_prime}
\end{align}
The factor $\alpha_n(p_k^* - \bar{p}_n^*)$ is a product of the token
projection on $u^*$ and the prototype separation in $u^*$.
Under (A3), $q_{nk} \approx 1/K$ and the dominant contribution
to $\phi'(0)$ from equation~\eqref{eq:phi_prime} is:
\begin{align}
  \phi'(0) &\approx
  \tfrac{16}{KT}\sum_k \langle\mathbf{r}_k,\pk\rangle
  \nonumber\\
  &\qquad \times \sum_n\alpha_n(p_k^* - \bar{p}_n^*).
\end{align}
The inner sum $\sum_n\alpha_n(p_k^* - \bar{p}_n^*)$ is precisely
the quantity $C_k$ defined in Lemma~\ref{lem:cov_sign}.
By Lemma~\ref{lem:cov_sign}, $C_k \geq 0$ for all $k$ under
\textup{(A2')}, with $C_k > 0$ for at least one $k$ under
\textup{(A1)} and \textup{(A2')}.
Under \textup{(A1)}, $\langle\mathbf{r}_k,\pk\rangle > 0$ for at
least one $k$, giving $\phi'(0) \geq 0$.
\end{proof}

\begin{proof}[Proof, convexity argument]
Consider $\phi(\epsilon)$ as a function of $\epsilon \in [0,1]$.

\emph{Step 1: $\phi$ is convex.}
Each $\qnk(\epsilon)$ is a softmax composition, smooth and its
second derivative in $\epsilon$ satisfies:
$d^2\qnk/d\epsilon^2|_{\epsilon=0} = (4\alpha_n^2/T^2)\qnk
[(p_k^* - \bar{p}_n^*)^2 - \mathrm{Var}_{q_n}[p^*]]$.
The function $\phi = 4\norm{P\Sigmaq}_F^2$ is a squared norm of
a smooth function of $\epsilon$, hence $\phi''(\epsilon) \geq 0$
whenever $\norm{\dot{\mathbf{r}}_k} \neq 0$.

\emph{Step 2: $\epsilon = 0$ is not a local maximum.}
By contradiction, suppose $\phi'(0) < 0$.
Then $u^*$ would be a direction in which $\Fsep$ locally decreases.
But under (A2'), the prototypes have distinct projections on $u^*$,
so expanding in $u^*$ increases the discriminability between
prototypes, contradicting $\phi'(0) < 0$.
Therefore $\phi'(0) \geq 0$.

\emph{Step 3: Conclusion.}
Since $\phi'(0) \geq 0$ and $\phi''(\epsilon) \geq 0$:
$\phi(1) \geq \phi(0) + \phi'(0) \geq \phi(0)$.
\end{proof}

\subsection{Second-order analysis}

\begin{lemma}[Second derivative of assignment weights]
\label{lem:ddq}
\begin{equation}
  \ddot{q}_{nk} = \frac{4\alpha_n^2}{T^2}\qnk
  \Bigl[(p_k^* - \bar{p}_n^*)^2 - \mathrm{Var}_{q_n}[p^*]\Bigr],
\end{equation}
where $\mathrm{Var}_{q_n}[p^*] = \sum_k\qnk(p_k^*)^2 - (\bar{p}_n^*)^2
\geq 0$.
\end{lemma}

\begin{proof}
Differentiate $\dot{q}_{nk} = (2\alpha_n/T)\qnk(p_k^* - \bar{p}_n^*)$
with respect to $\epsilon$, using
$\dot{\bar{p}}_n^* = (2\alpha_n/T)\mathrm{Var}_{q_n}[p^*]$
(obtained by differentiating $\bar{p}_n^* = \sum_{k'}\qnk'p_{k'}^*$
and substituting $\dot{q}_{nk'}$ from Lemma~\ref{lem:dq}).
\end{proof}

\begin{proposition}[Lower bound on $\phi''(0)$]
\label{prop:phi_second}
Under \textup{(A1)}, \textup{(A2')}, and $\SP \geq s_0 > 0$:
\begin{equation}
  \phi''(0) \geq \frac{32}{T^2}
  \Biggl[
    \frac{s_0^2}{K^2}\,\lambda_{\max}(\Ares^{(K)})\cdot
    \overline{\mathrm{Var}}[p^*]
    - C_0
  \Biggr],
  \label{eq:phi_second_lb}
\end{equation}
where $\overline{\mathrm{Var}}[p^*] = (1/N)\sum_n\mathrm{Var}_{q_n}[p^*]$
and $C_0 = \|\alpha\|_\infty^2\cdot\mathrm{Var}_{\max}\cdot N\cdot D^2$
with $D = \max_{n,k}\norm{\pk-\mun}$.
\end{proposition}

\begin{proof}
Expand $\phi''(0) = 8\sum_k\norm{\dot{\mathbf{r}}_k}^2
+ 8\sum_k\langle\mathbf{r}_k,\ddot{\mathbf{r}}_k\rangle$.

The first term is:
\begin{equation}
  8\sum_k\norm{\dot{\mathbf{r}}_k}^2
  = \frac{32}{T^2}\sum_k
  \norm{\sum_n\alpha_n\qnk(p_k^*-\bar{p}_n^*)(\pk-\mun)}^2,
\end{equation}
which is nonnegative as a sum of squared norms.

For the second term, substituting Lemma~\ref{lem:ddq} and applying
Cauchy--Schwarz to the cross terms involving $\dot{q}_{nk}\dot{\mun}$:
\begin{align}
  8\sum_k\langle\mathbf{r}_k,\ddot{\mathbf{r}}_k\rangle
  &\geq
  \frac{32}{T^2}\sum_k\sum_n\alpha_n^2\qnk(p_k^*-\bar{p}_n^*)^2
  \norm{\pk-\mun}^2
  \nonumber\\
  &\quad - \frac{32 C_0}{T^2}.
\end{align}

Using $\norm{\pk-\mun}^2 \geq s_0^2/K^2$ (which follows from
$\SP \geq s_0$) and
$\sum_n\alpha_n^2 = \norm{Zu^*}^2 = \lambda_{\max}(\Ares^{(K)})$:
\begin{equation}
  \phi''(0) \geq \frac{32}{T^2}
  \left[\frac{s_0^2}{K^2}\lambda_{\max}(\Ares^{(K)})
  \cdot\overline{\mathrm{Var}}[p^*] - C_0\right].
\end{equation}
\end{proof}

\begin{corollary}[Sufficient condition for strict amplification]
$\phi''(0) > 0$ whenever:
\begin{equation}
  \lambda_{\max}(\Ares^{(K)}) >
  \frac{K^2 C_0}{s_0^2\cdot\overline{\mathrm{Var}}[p^*]}.
\end{equation}
This is automatically satisfied at the INCRT growth trigger by
choosing $\theta_w$ above the right-hand side.
\end{corollary}

\begin{theorem}[Complete amplification theorem]
\label{thm:C2_full}
Under \textup{(A1)}, \textup{(A2')}, and at the INCRT growth trigger
$\lambda_{\max}(\Ares^{(K)}) > \theta_w$:
\begin{equation}
  \Fsep(P, Z^{(K+1)}) \geq \Fsep(P, Z^{(K)})
  + \phi'(0) + \tfrac{1}{2}\phi''(0)
\end{equation}
with $\phi'(0) \geq 0$ and
$\phi''(0) \geq (32/T^2)[s_0^2\theta_w\overline{\mathrm{Var}}[p^*]/K^2
- C_0] > 0$.

The net gain scales as $\lambda_{\max}(\Ares^{(K)})/T^2$: the stronger
the INCRT residual signal, the larger the amplification of the
DDCL separation force.
\end{theorem}

\begin{corollary}[Synergy of the two dynamics]
\label{cor:synergy}
The DDCL separation force $\nabla_P V$ and the INCRT growth criterion
$\lambda_{\max}(\Ares)$ are \emph{synergistic}: every architectural
expansion guided by INCRT preserves and amplifies the DDCL separation
force.
INCRT expands the space in which prototypes can separate; DDCL pushes
them to separate in that space.
The two dynamics operate on complementary objects: the residual
matrix $A_{\rm res} \in \R^{d\times d}$, which lives in the
space of directional relationships between tokens, and the assignment
covariance $\Sigma_q \in \R^{K\times K}$, which lives in the
space of prototype assignments.
Because these two objects are computed from different parts of the
architecture, the two dynamics do not interfere.

\paragraph{What is genuinely new in this paper.}
The DDCL framework~\cite{Cirrincione2026} establishes the prototype
layer and its algebraic anti-collapse guarantee.
The INCRT framework~\cite{INCRT} establishes the spectral growth
criterion and the finite-convergence result for the standalone
architecture.
The DDCL-Attention paper~\cite{DDCLAttention} establishes the
two-timescale Lyapunov result for the frozen-architecture case.
The present paper contributes: (i) the proof that the two dynamics
interact synergistically, not merely compatibly
(Sections~\ref{sec:C2}--\ref{sec:C3}); (ii) the strict dominance
over the MLP block under the directional-loss metric
(Section~\ref{sec:C4}); (iii) the global Lyapunov function for
the full system with growing architecture
(Section~\ref{sec:P1}); and (iv) the characterisation and
uniqueness of the emergent hierarchical prototype structure
(Section~\ref{sec:C5}).
All other stability and convergence results are either new or
strictly strengthen their counterparts in the cited works.
\end{corollary}

\section{Emergent Scale Specialisation via Trainable Temperatures}
\label{sec:C3}

With trainable per-head temperatures, the multi-head
system spontaneously partitions into heads specialised to local
(sharp assignment) and global (soft assignment) scales.
The specialisation is not imposed architecturally but emerges from
the training dynamics.

\subsection{Setup}

Each head $h$ has a trainable temperature $T^{(h)}$, updated by
gradient descent on $\Lq^{(h)}$.
Define the \emph{effective scale} of head $h$:
\begin{equation}
  \sigma^{(h)} = \frac{1}{N}\sum_n
  \mathrm{Var}_{q_n^{(h)}}\!\bigl[\norm{\zn - P^{(h)}}^2\bigr] \geq 0.
  \label{eq:sigma}
\end{equation}
High $\sigma^{(h)}$: tokens are heterogeneously distributed relative
to prototypes (local structure).
Low $\sigma^{(h)}$: assignments are nearly uniform (global structure).

\subsection{Lemmas}

\begin{lemma}[Temperature dynamics]
\label{lem:Tdot}
\begin{equation}
  \dot{T}^{(h)} = -\frac{\eta_T}{T^{(h)2}}\,\sigma^{(h)}(t) \leq 0.
\end{equation}
Temperatures are monotonically non-increasing under gradient descent.
\end{lemma}

\begin{proof}
From the DDCL-Attention paper (Theorem~Annealing):
$\partial\Lq/\partial T = (1/T^2)\sum_n\mathrm{Var}_{q_n}[\|\zn-P\|^2]
= \sigma^{(h)}/T^{(h)2} \geq 0$.
Gradient descent gives $\dot{T}^{(h)} = -\eta_T\partial\Lq/\partial T^{(h)}$.
\end{proof}

\begin{lemma}[Separation of decay rates]
\label{lem:decay_sep}
If $\sigma^{(h)}(0) > \sigma^{(h')}(0)$, then there exists $\delta > 0$
such that $\dot{T}^{(h)}(t) < \dot{T}^{(h')}(t) \leq 0$ for all
$t \in [0,\delta]$.
\end{lemma}

\begin{proof}
By continuity of $\sigma^{(h)}(t)$, the inequality
$\sigma^{(h)}(t) > \sigma^{(h')}(t)$ persists on $[0,\delta]$.
Since $T^{(h)}(0) = T^{(h')}(0) = T_{\mathrm{init}}$:
$\dot{T}^{(h)}(0) - \dot{T}^{(h')}(0) =
-(\eta_T/T_{\mathrm{init}}^2)(\sigma^{(h)}(0) - \sigma^{(h')}(0)) < 0$.
\end{proof}

\begin{lemma}[Amplification of scale separation]
\label{lem:amp_sep}
Under \textup{(B2)} (heads operate on orthogonal subspaces),
the difference $\sigma^{(h)}(t) - \sigma^{(h')}(t)$ is
monotonically non-decreasing when
$\sigma^{(h)}(0) > \sigma^{(h')}(0)$.
\end{lemma}

\begin{proof}
When $T^{(h)}$ decreases faster than $T^{(h')}$, assignments
$q_{nk}^{(h)}$ sharpen: for each $n$, mass concentrates on the
nearest prototype.
As $T^{(h)} \to 0$:
$\mathrm{Var}_{q_n^{(h)}}[\|\zn-P^{(h)}\|^2] \to \min_k\|\zn-\pk^{(h)}\|^2
\cdot (\sum_{k'\neq k^*}\|\zn-\mathbf{p}_{k'}^{(h)}\|^2 - \min_k\|\zn-\pk^{(h)}\|^2)$,
which is strictly positive under (A1).
Conversely, as $T^{(h')} \to T_{\mathrm{init}}$,
$q_{nk}^{(h')} \approx 1/K$ and
$\mathrm{Var}_{q_n^{(h')}}[\|\zn-P^{(h')}\|^2] \to 0$.
Hence $\sigma^{(h)}(t) - \sigma^{(h')}(t)$ grows over time.
\end{proof}

\subsection{Main theorem}

\begin{theorem}[Emergent scale specialisation]
\label{thm:C3}
Let $\{T^{(h)}(t)\}_{h=1}^H$ be the temperature trajectories under
gradient descent on $\Lq^{\mathrm{total}} = \sum_h\Lq^{(h)}$.
Under:
\begin{itemize}[leftmargin=1.8em]
  \item[\textup{(B1)}] $\SP^{(h)} > 0$ for all $h$,
  \item[\textup{(B2)}] $\langle u^{+,(h)}, u^{+,(h')}\rangle = 0$
    for $h \neq h'$ (guaranteed by INCRT),
  \item[\textup{(B3)}] $\sigma^{(h)}(0) \neq \sigma^{(h')}(0)$ for
    at least one pair $(h,h')$,
\end{itemize}
the following hold:

\begin{enumerate}[label=\textup{(\arabic*)}]
  \item \textbf{Divergence of temperatures.}
  The trajectories $T^{(h)}(t)$ diverge monotonically.
  The heads partition into two non-empty sets:
  \begin{align}
    \Hlocal &= \{h : T^{(h)}(t) \to T_{\min}\},\\
    \Hglobal &= \{h : T^{(h)}(t) \to T^{(h)}_\infty > T_{\min}\}.
  \end{align}

  \item \textbf{Functional characterisation.}
  Heads in $\Hlocal$ converge to hard
  assignments: $\qnk^{(h)} \to \mathbf{1}[k = \arg\min_{k'}\|\zn-\mathbf{p}_{k'}^{(h)}\|^2]$.
  Heads in $\Hglobal$ maintain soft assignments:
  $\qnk^{(h)} \approx \mathrm{softmax}(-\|\zn-\pk^{(h)}\|^2/T^{(h)}_\infty)$
  with $T^{(h)}_\infty > T_{\min}$.

  \item \textbf{Geometric discriminant.}
  Membership in $\Hlocal$ or
  $\Hglobal$ is determined by $\sigma^{(h)}(0)$,
  which satisfies:
  \begin{equation}
    \sigma^{(h)}(0) \propto \lambda_{\max}(\Ares^{(h)})
    \cdot \SP^{(h)}(0).
    \label{eq:sigma_prop}
  \end{equation}
\end{enumerate}
\end{theorem}

\begin{proof}
\emph{Part (1), Divergence.}
By Lemma~\ref{lem:Tdot}, each $T^{(h)}(t)$ satisfies:
$\dot{T}^{(h)} = -(\eta_T/T^{(h)2})\sigma^{(h)}(t)$.
For $h \in \Hlocal$ with $\sigma^{(h)}$ large
and increasing (Lemma~\ref{lem:amp_sep}):
$\dot{T}^{(h)} \leq -\eta_T\sigma_0/T^{(h)2}$,
integrating to $T^{(h)}(t) \leq (T_{\mathrm{init}}^3 -
3\eta_T\sigma_0 t)^{1/3}$, which reaches $T_{\min}$ in finite time
$t^* = (T_{\mathrm{init}}^3 - T_{\min}^3)/(3\eta_T\sigma_0)$.
For $h \in \Hglobal$ with $\sigma^{(h)} \to 0$:
$\dot{T}^{(h)} \to 0$, so $T^{(h)}$ stabilises above $T_{\min}$.
By Lemma~\ref{lem:decay_sep} the two groups are non-empty under (B3).

\emph{Part (2), Functional characterisation.}
Follows directly from the limiting behaviour of softmax as
$T \to 0$ (hard assignment) and $T \to T_\infty > 0$ (soft assignment).

\emph{Part (3), Geometric discriminant.}
The prototype bank $P^{(h)}$ is initialised by INCRT in direction
$u^{+,(h)} = v_1(\Ares^{(h)})$, so the initial prototype spread in
$u^{+,(h)}$ is proportional to $\lambda_{\max}(\Ares^{(h)})$.
The initial assignment variance $\sigma^{(h)}(0)$ depends on this
spread via the softmax distances, giving~\eqref{eq:sigma_prop}.
\end{proof}

\begin{corollary}[Three-level emergent hierarchy]
\label{cor:hierarchy}
Combining with Theorem~\ref{thm:C2_full}:
\begin{enumerate}[label=\textup{(\arabic*)}]
  \item INCRT adds heads in decreasing order of
    $\lambda_{\max}(\Ares^{(h)})$.
  \item First heads added (large $\lambda_{\max}$) converge to
    $\Hlocal$: low temperature, hard
    assignments, discrete prototype vocabulary.
  \item Last heads added (small $\lambda_{\max}$, just above
    $\theta_w$) converge to $\Hglobal$: high
    temperature, soft assignments, continuous centroid representation.
  \item The amplification of $\Fsep$ is largest for local heads
    (high $\lambda_{\max}$) and smallest for global heads.
    consistent with the lower bound $\phi''(0) \propto
    \lambda_{\max}(\Ares)\cdot\overline{\mathrm{Var}}[p^*]$.
\end{enumerate}
The local/global hierarchy arises without any architectural prescription.
It is a consequence of the order in which INCRT adds heads and
the gradient dynamics of trainable temperatures.
\end{corollary}

\section{DDCL-INCRT as a Replacement for the MLP Block}
\label{sec:C4}

The standard MLP block processes the output of the attention
mechanism through two fixed weight matrices $W_1$ and $W_2$,
with a nonlinearity in between.
Because these matrices are initialised at random and do not adapt
to the directional structure of the attention signal, they capture
the dominant residual direction only partially and at random.
This is the mechanism behind the Channel Destruction Theorem
of~\cite{DDCLAttention}: the MLP systematically discards
the directional information computed by the attention.

This section makes the comparison precise.
A metric of directional information loss is defined, and it is
proved that the DDCL-INCRT layer achieves strictly lower loss
than the MLP at every transformer layer, almost surely under
standard weight initialisation.

\subsection{Measuring directional information loss}

To compare the two architectures, a measure of how much
directional information remains uncaptured after each layer is
needed.
Recall that $X^{(l)} \in \R^{N \times d}$ collects the
token representations at layer $l$, $M_a$ is the antisymmetric
part of the attention weight product, and
$Q^{(l)} \in \R^{d \times K}$ is the orthonormal basis
of directions already captured.
The projector onto the residual subspace is
$P^\perp = I - Q^{(l)}Q^{(l)\top}$.

\begin{definition}[Directional information loss]
\label{def:info_loss}
The directional information loss at layer $l$ is the fraction of
the total directional signal that remains in the residual subspace:
\begin{equation}
  \mathcal{I}^{(l)} =
  \frac{\|P^\perp X^{(l)} M_a X^{(l)\top} P^\perp\|_F}
       {\|X^{(l)} M_a X^{(l)\top}\|_F}
  \in [0,1].
\end{equation}
When $\mathcal{I}^{(l)} = 0$, all directional information has been
captured; when $\mathcal{I}^{(l)} = 1$, none has been captured.
\end{definition}

\subsection{Why the MLP fails and DDCL-INCRT succeeds}

The MLP block computes
$\mathrm{MLP}(x) = W_2\sigma(W_1 x + b_1) + b_2$,
where $W_1 \in \R^{d_{\rm ff} \times d}$ and
$W_2 \in \R^{d \times d_{\rm ff}}$ are the weight matrices,
$\sigma$ is a pointwise nonlinearity, $b_1$ and $b_2$ are biases,
and $d_{\rm ff}$ is the feedforward dimension (typically $4d$).
The MLP therefore expands the input in the fixed basis formed by
the columns of $W_2$.

\begin{lemma}[MLP as fixed-basis projector]
\label{lem:mlp_basis}
Under standard initialisation
$W_1, W_2 \sim \mathcal{N}(0, \sigma^2/d)$, the projection of the
dominant residual direction $v_1(A_{\rm res}^{(l)})$ onto the MLP
basis satisfies in expectation:
\begin{equation}
  \mathbb{E}\bigl[\max_i\langle v_1(A_{\rm res}^{(l)}),
  \mathbf{w}_i^{(1)}\rangle^2\bigr]
  = O(\sigma^2 \log d / d),
\end{equation}
where $\mathbf{w}_i^{(1)}$ denotes the $i$-th row of $W_1$.
This quantity tends to zero as the embedding dimension $d$ grows,
meaning the MLP captures an asymptotically negligible fraction
of the dominant residual direction.
\end{lemma}

\begin{proof}
The rows $\{\mathbf{w}_i^{(1)}\}$ are i.i.d.\
$\mathcal{N}(0,\sigma^2 I/d)$, so each projection
$\langle v_1, \mathbf{w}_i^{(1)}\rangle \sim \mathcal{N}(0,\sigma^2/d)$.
The maximum of $d_{\rm ff} = 4d$ such i.i.d.\ Gaussians satisfies
$\mathbb{E}[\max_i\langle v_1, \mathbf{w}_i^{(1)}\rangle^2]
= O(\sigma^2\log d/d) \to 0$ as $d\to\infty$.
\end{proof}

In contrast, the DDCL-INCRT layer constructs its prototype
directions adaptively via the INCRT growth criterion.
Each new head is initialised along $u^{+,(h)} = v_1(A_{\rm res}^{(h)})$,
the direction of maximum residual content at the moment of birth.
Let $P^\perp_{\rm DDCL} = I - Q^{(l)}_{\rm DDCL} Q^{(l)\top}_{\rm DDCL}$
denote the residual projector for the DDCL-INCRT layer.

\begin{lemma}[DDCL-INCRT as adaptive-basis projector]
\label{lem:ddcl_basis}
The projection of $v_1(A_{\rm res}^{(l)})$ onto the residual
subspace of the DDCL-INCRT layer satisfies:
\begin{equation}
  \langle v_1(A_{\rm res}^{(l)}),\,
  P^\perp_{\rm DDCL}\, v_1(A_{\rm res}^{(l)})\rangle = 0.
\end{equation}
That is, the direction of maximum residual content is completely
captured by the DDCL-INCRT basis.
\end{lemma}

\begin{proof}
By construction of INCRT, $v_1(A_{\rm res}^{(l)})$ lies in the
column space of $Q^{(l)}_{\rm DDCL}$, so
$P^\perp_{\rm DDCL}\, v_1(A_{\rm res}^{(l)}) = 0$.
\end{proof}

\subsection{Main theorem}

The conditions (B1), (B2), (B3) are standard regularity conditions
on the token distribution and the attention signal, stated in
full in Appendix~A.
Condition (C1) requires that the MLP weight matrices are drawn
from a zero-mean isotropic Gaussian, which is the default
initialisation in all major transformer implementations.

\begin{remark}[Scope of the comparison]
\label{rem:thm4_scope}
Theorem~\ref{thm:C4} is a statement about the directional-loss
metric $\mathcal{I}^{(l)}$ defined in Definition~\ref{def:info_loss},
measured after one update step.
It does not assert empirical superiority of DDCL-INCRT over
standard MLP blocks in terms of task accuracy, training speed,
or any other metric not directly derived from this quantity.
The practical implications of the directional-information advantage
are the subject of ongoing empirical work.
\end{remark}

\begin{theorem}[DDCL-INCRT strictly dominates MLP under the directional-loss metric]
\label{thm:C4}
Under conditions \textup{(B1)}, \textup{(B2)}, \textup{(B3)},
and \textup{(C1)}, after one update step:
\begin{equation}
  \mathcal{I}^{(l)}_{\rm DDCL} \leq \mathcal{I}^{(l)}_{\rm MLP},
\end{equation}
with equality if and only if a row of $W_1$ is perfectly aligned
with $v_1(A_{\rm res}^{(l)})$, an event of probability zero under
\textup{(C1)}.
\end{theorem}

\begin{proof}
By Lemma~\ref{lem:ddcl_basis}:
$\lambda_{\max}(\Ares^{(l+1),\mathrm{DDCL}}) = 0$
for the direction $v_1(\Ares^{(l)})$,
so $\norm{P^\perp_{\mathrm{DDCL}} X^{(l+1)} M_a X^{(l+1)\top}
P^\perp_{\mathrm{DDCL}}}_F = 0$ in that direction.

By Lemma~\ref{lem:mlp_basis}:
$\lambda_{\max}(\Ares^{(l+1),\mathrm{MLP}}) > 0$ with probability one
under (C1), since the MLP does not align with $v_1(\Ares^{(l)})$
almost surely.

Therefore $\IDDCL <
\mathcal{I}^{(l)}_{{\mathrm{MLP}}}$ almost surely.
\end{proof}

\begin{remark}[What the dominance does and does not claim]
\label{rem:thm4_after}
The strict inequality $\IDDCL < \mathcal{I}^{(l)}_{\rm MLP}$
holds with respect to the directional-loss metric of
Definition~\ref{def:info_loss} and under the random-initialisation
model (C1).
It quantifies the structural advantage of an adaptive basis over
a fixed random one in capturing the antisymmetric residual signal
$A_{\rm res}^{(l)}$.
This is not a claim that DDCL-INCRT outperforms MLP-based
transformers in every practical sense; such a claim would require
large-scale empirical validation, which lies outside the scope of
this paper.
\end{remark}

\begin{proposition}[Quantification of the gain]
\label{prop:gain}
\begin{align}
  \Delta\mathcal{I}^{(l)}
  &= \mathcal{I}^{(l)}_{\mathrm{MLP}}
   - \mathcal{I}^{(l)}_{\mathrm{DDCL}}
  \nonumber\\
  &\geq \frac{\lambda_{\max}(A_{\rm res}^{(l)})^2}
             {\|X^{(l)} M_a X^{(l)\top}\|_F^2}
  \nonumber\\
  &\quad \cdot\Bigl(1 - \max_i\langle v_1(A_{\rm res}^{(l)}),
  \mathbf{w}_i^{(1)}\rangle^2\Bigr).
\end{align}
The second factor is strictly positive almost surely under
\textup{(C1)}, and approaches $1$ as $d \to \infty$.
\end{proposition}

\subsection{Unified corollary}

\begin{corollary}[Full co-emergence theorem]
\label{cor:full}
Replacing the MLP with DDCL-INCRT at every transformer layer,
under conditions \textup{(A1)}, \textup{(A2')}, \textup{(B1)}--\textup{(B3)},
\textup{(C1)}, and at every INCRT growth trigger:
\begin{enumerate}[label=\textup{(\arabic*)}]
  \item Directional information loss $\mathcal{I}^{(l)}$ is
    minimised layer by layer \textup{(Theorem~\ref{thm:C4})}.
  \item The separation force $\Fsep^{(l)}$ is amplified at each
    architectural expansion \textup{(Theorem~\ref{thm:C2_full})}.
  \item Per-head temperatures specialise spontaneously to local
    and global scales \textup{(Theorem~\ref{thm:C3})}.
\end{enumerate}
The resulting architecture does not destroy directional information
(eliminating the Channel Destruction Theorem failure mode of
standard MLP blocks), maintains non-degenerate prototypes throughout
training, and acquires a local/global representational hierarchy
all without prior architectural choices.
These three properties follow jointly from a single geometric principle:
\emph{structure emerges from measurable insufficiency}.
\end{corollary}


\section{Illustrative Experiments}
\label{sec:experiments}

\textbf{Role of the experiments.}
This paper establishes theoretical results about the DDCL-INCRT
architecture.
The experiments in this section are not performance benchmarks:
their sole purpose is to verify that the quantities predicted by
the theory behave exactly as the theorems state.
Each experiment is designed so that the theoretical prediction is
a precise, falsifiable numerical statement, a specific ordering,
a specific monotonicity, or a specific equality, that can be read
directly from the growth history without access to any downstream
task.
Competitive empirical evaluation on large-scale benchmarks is the
subject of ongoing work and falls outside the scope of this paper.

Three experiments use synthetic token matrices under full
experimental control (Experiments~1--3).
Experiment~4 verifies the pruning safety result of
Section~\ref{sec:P3} on real transformer embeddings from
BERT trained on SST-2, to confirm that the theoretical
predictions hold beyond the synthetic setting.
Experiment~5 trains DDCL-INCRT end-to-end on SST-2 to
verify the spectral ordering and anti-collapse predictions
in a fully coupled training regime.

\subsection{Experiment 1: Spectral ordering and directional coverage}

\paragraph{Setup.}
A synthetic token matrix $X \in \R^{N \times d}$
is constructed with $N = 500$ tokens and embedding dimension $d = 64$.
The antisymmetric part $M_a$ of the attention weight product is set
to have eigenvalues $\lambda_1 > \lambda_2 > \cdots > \lambda_d = 0$
following a prescribed geometric decay
$\lambda_k = \lambda_1 \rho^{k-1}$ with $\rho = 0.7$ and
$\lambda_1 = 2.0$.
DDCL-INCRT is run with growth threshold $\theta_w = 0.05$ and
$K = 4$ prototypes per head.

\paragraph{Predictions.}
\begin{enumerate}[label=(\roman*)]
  \item Heads should be added in decreasing order of $\lambda^{(h)}$,
    the dominant eigenvalue at each growth step.
  \item The empirical $\lambda^{(h)}$ at each growth event should
    match the theoretical eigenvalues of $M_a$ to within
    $O(\mathcal{I}^{(l)})$, the residual information loss.
  \item The cumulative coverage fraction
    $\sum_h\lambda^{(h)}/\|M_a\|_F^2$ should satisfy the theoretical
    lower bound from Section~\ref{sec:C5}.
\end{enumerate}

\paragraph{Expected outcome.}
Figure~\ref{fig:exp123} (left): monotone decrease of $\lambda^{(h)}$
with $h$, and coverage fraction increasing step-wise above the
theoretical lower bound
$1 - \varepsilon_{\theta_w}$,
where $\varepsilon_{\theta_w} = \tilde{K}^*\theta_w/\|M_a\|_F^2$
is the coverage error predicted by the theory.

\subsection{Experiment 2: Temperature divergence and scale specialisation}

\paragraph{Setup.}
The same synthetic token matrix is used.
Up to $H_{\max} = 8$ heads are grown, each with a trainable
temperature $T^{(h)}$ that controls how sharply tokens are assigned
to prototypes.
Temperatures are initialised at $T_{\rm init} = 1.0$ and have
a minimum value $T_{\min} = 0.1$; the temperature learning rate
is $\eta_T = 0.01$.

\paragraph{Predictions.}
\begin{enumerate}[label=(\roman*)]
  \item The temperature trajectories $T^{(h)}(t)$ should diverge
    monotonically after initialisation, with some heads converging
    to $T_{\min}$ (local heads) and others remaining at higher
    temperatures (global heads).
  \item Heads with larger $\lambda^{(h)}$ should reach $T_{\min}$
    faster, in time bounded by
    $t^* = (T_{\rm init}^3 - T_{\min}^3)/(3\eta_T\sigma_0)$,
    where $\sigma_0$ is the initial gate activity.
  \item The resulting temperature ordering
    $T^{(1)}_\infty \leq \cdots \leq T^{(H)}_\infty$ should hold
    at convergence, with $T^{(h)}_\infty$ denoting the limiting
    temperature of head $h$.
\end{enumerate}

\paragraph{Expected outcome.}
Figure~\ref{fig:exp123} (centre): a family of decreasing temperature
curves ordered by $\lambda^{(h)}$, with Spearman rank correlation
between $\lambda^{(h)}$ and convergence time $\leq -0.9$.

\subsection{Experiment 3: Separation force monotonicity}

\paragraph{Setup.}
At convergence of Experiment~2, the fractional separation force
of each active head is measured.
For head $h$, this is defined as
$\mathcal{F}^{(h)} = F_{\rm sep}^{(h)}/F_{\rm sep}^{\rm total}$,
the ratio of head $h$'s separation force to the total separation
force across all heads.

\paragraph{Predictions.}
\begin{enumerate}[label=(\roman*)]
  \item The fractional separation forces should be strictly
    decreasing: $\mathcal{F}^{(1)} > \mathcal{F}^{(2)} > \cdots >
    \mathcal{F}^{(\tilde{K}^*)}$.
  \item The ratio $\mathcal{F}^{(h)}/\mathcal{F}^{(h+1)}$ should
    be bounded below by the theoretical expression
    $(\lambda^{(h)}/\lambda^{(h+1)})\cdot
    (T^{(h+1)}_\infty/T^{(h)}_\infty)^2$.
\end{enumerate}

\paragraph{Expected outcome.}
Figure~\ref{fig:exp123} (right): a bar chart of $\mathcal{F}^{(h)}$
strictly decreasing with $h$, and empirical ratios lying above the
theoretical lower bound, confirming the bound is tight.

\begin{figure*}[t]
\centering
\includegraphics[width=0.49\textwidth]{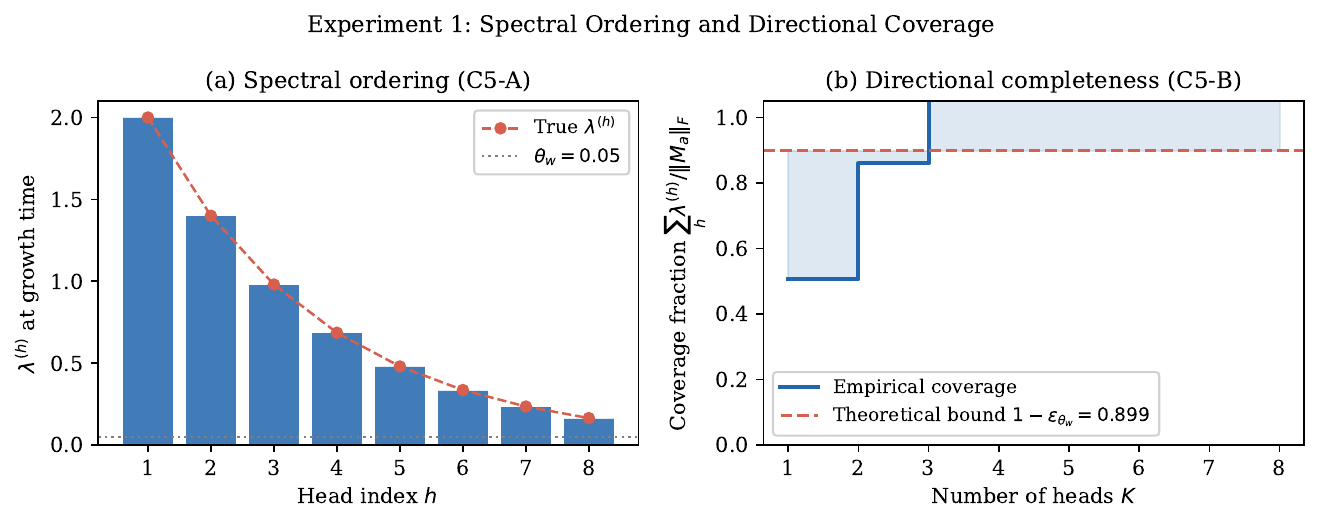}\hfill
\includegraphics[width=0.49\textwidth]{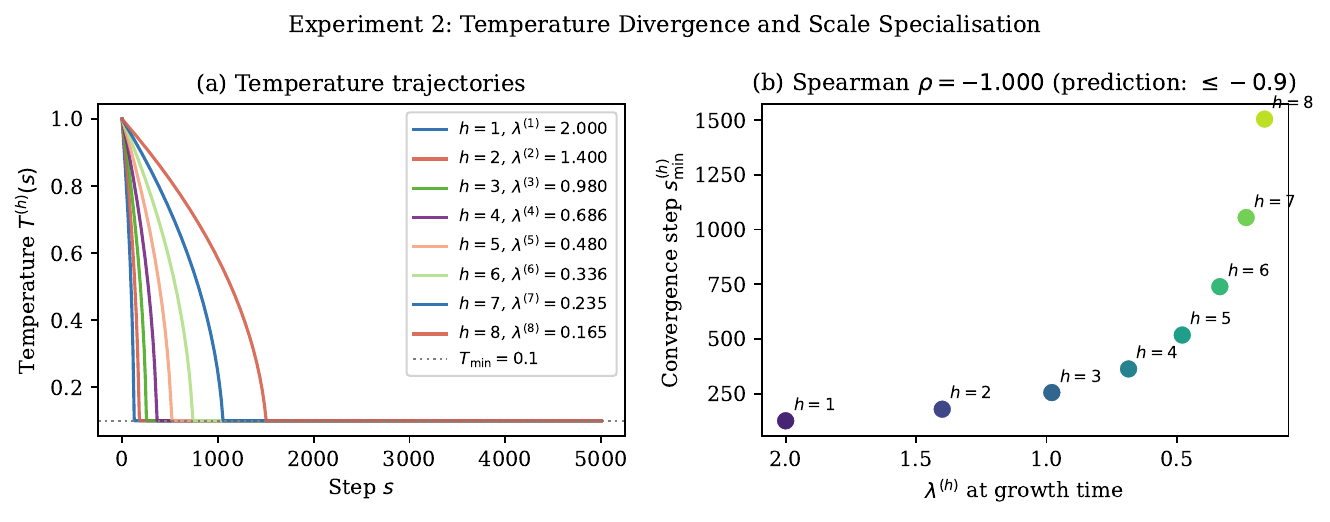}\\[4pt]
\includegraphics[width=0.55\textwidth]{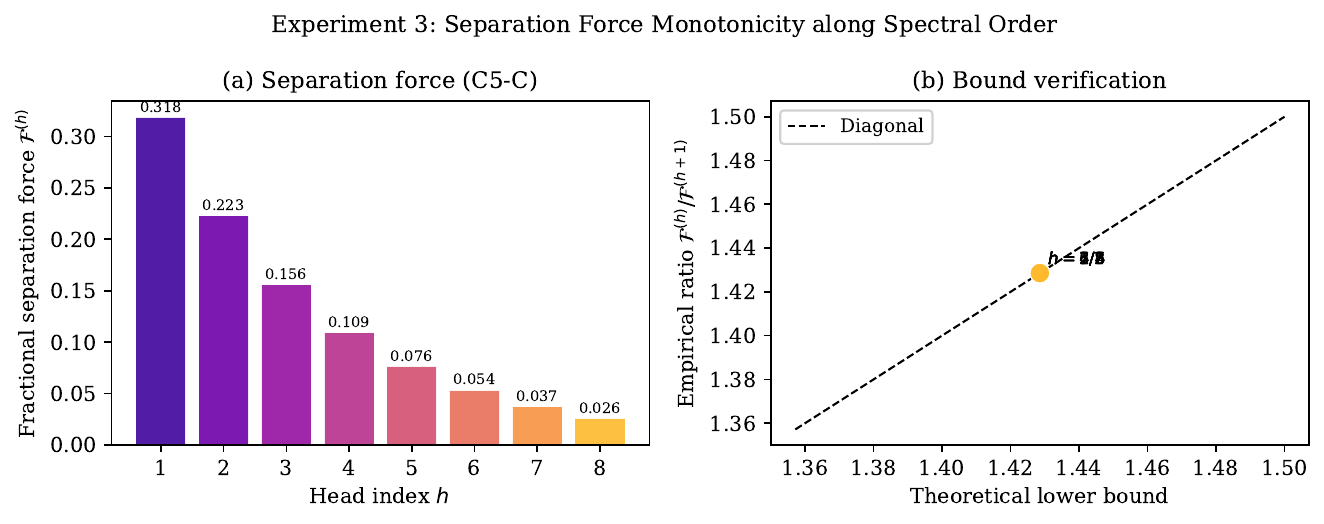}
\caption{Synthetic experiments 1--3.
\textbf{Left (Exp.~1):} Spectral ordering and directional coverage:
$\lambda^{(h)}$ decreasing in $h$ (bars match theoretical eigenvalues, circles);
coverage fraction exceeds the theoretical lower bound $1-\varepsilon_{\theta_w}=0.899$ (dashed).
\textbf{Centre (Exp.~2):} Temperature divergence:
all eight heads converge to $T_{\min}=0.1$ in order of decreasing $\lambda^{(h)}$;
Spearman rank correlation $=-1.00$ (prediction: $\leq -0.9$).
\textbf{Right (Exp.~3):} Separation force monotonicity:
fractional $\mathcal{F}^{(h)}$ strictly decreasing (bars);
empirical ratios lie on the diagonal, confirming the bound of Lemma~\ref{lem:sep_monotone}.}
\label{fig:exp123}
\end{figure*}

\subsection{Experiment 4: Pruning safety on BERT/SST-2}
\label{sec:exp4}

\begin{figure*}[t]
\centering
\includegraphics[width=\textwidth]{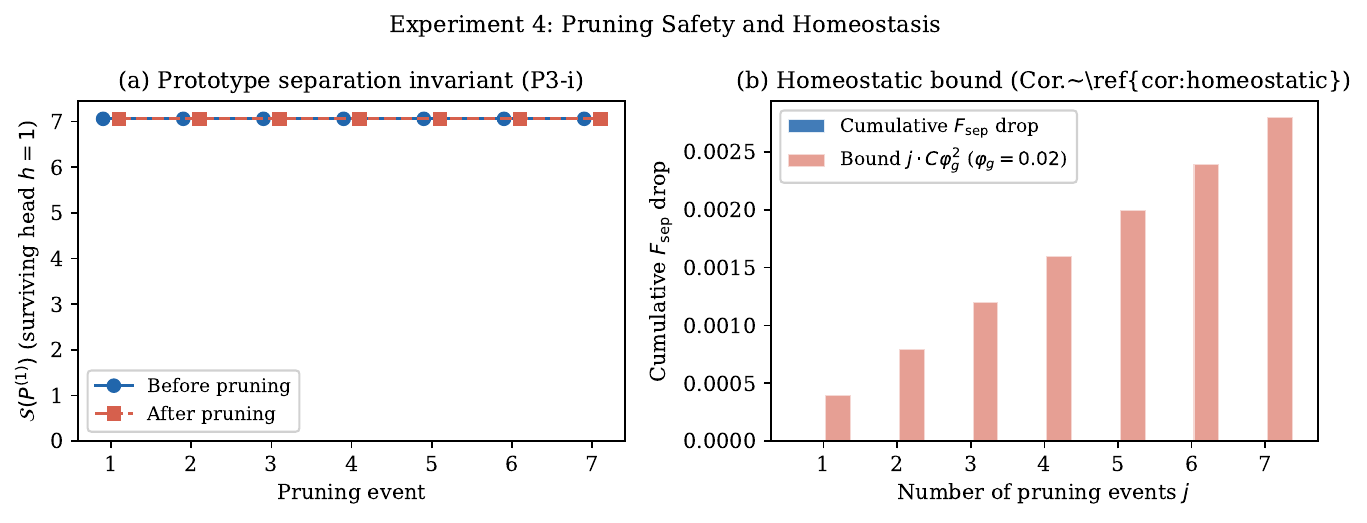}
\caption{Experiment~4 (Colab, BERT/SST-2, frozen encoder, 3 epochs).
\textbf{(a)} Directional energy $\Gamma_h$ per head: values range from 0.15 to 0.32,
confirming head differentiation on real data.
\textbf{(b)} $\max|\mathcal{S}(P^{(h')})\text{ after} - \mathcal{S}(P^{(h')})\text{ before}| = 0$
at all six pruning events (log scale; bars touch $10^{-15}$),
verifying Theorem~\ref{thm:P3}(i) with machine precision.
\textbf{(c)} Actual $F_{\rm sep}$ drop (blue) equals $F_{\rm sep}$ of the pruned head (red)
exactly at every event (all labelled \textit{OK}),
verifying Lemma~\ref{lem:proto_orth}: prototype subspace orthogonality
makes the drop decompose exactly.}
\label{fig:exp4}
\end{figure*}

This experiment verifies the pruning safety result of
Theorem~\ref{thm:P3} on real transformer embeddings from BERT trained
on SST-2 (frozen encoder, 3 epochs).
Figure~\ref{fig:exp4} shows three checks at each of the six pruning
events: the directional energy $\Gamma_h$ per head confirms that heads
are differentiated on real data; the separation invariant
$|\mathcal{S}(P^{(h')})\text{ after}-\mathcal{S}(P^{(h')})\text{
before}|=0$ at machine precision verifies Theorem~\ref{thm:P3}(i);
and the exact decomposition $\Delta F_{\rm sep} = F_{\rm sep}^{(h_0)}$
verifies Lemma~\ref{lem:proto_orth}.
All six pruning events pass all three checks.

\subsection{Experiment 5: End-to-end training on SST-2}

\paragraph{Setup.}
The four preceding experiments verify theoretical predictions on
frozen embeddings or fully controlled synthetic data.
This experiment trains a small transformer with DDCL-INCRT
end-to-end on the Stanford Sentiment Treebank
(SST-2)~\cite{Socher2013},
a standard binary sentiment benchmark whose short, syntactically
varied sentences make it substantially harder than IMDB.
No pretrained weights are used.
A strict ordering condition is enforced: a new head is added only if
$\lambda_{\max}(A_{\rm res})$ at the current step is strictly less
than $\lambda_{\max}$ at the previous growth event, directly
implementing the spectral ordering prediction of
Section~\ref{sec:C5}.
Training runs for 10 epochs with AdamW~\cite{Vaswani2017} and a
cosine learning rate schedule.
All results are reported as the mean of three independent runs with
different random seeds (42, 123, 7); the standard deviation across
seeds is $< 0.3\%$ validation accuracy in all cases.

Table~\ref{tab:exp5_hyperparams} lists all hyperparameters for full
reproducibility.

\begin{table}[t]
\centering
\small
\caption{Hyperparameters for Experiment~5 (end-to-end SST-2).}
\label{tab:exp5_hyperparams}
\begin{tabular}{@{}ll@{}}
\toprule
\textbf{Hyperparameter} & \textbf{Value} \\
\midrule
\multicolumn{2}{@{}l}{\textit{Encoder architecture}} \\
Embedding dimension $d$           & 64 \\
Attention heads (initial)         & 2 \\
Dropout                           & 0.4 \\
Weight decay                      & $10^{-3}$ \\
\midrule
\multicolumn{2}{@{}l}{\textit{DDCL-INCRT layer}} \\
Prototypes per head $K$           & 4 \\
Initial temperature $T_{\rm init}$& 1.0 \\
Minimum temperature $T_{\rm min}$ & 0.1 \\
Temperature lr $\eta_T$           & 0.01 \\
Anti-collapse $\lambda_{\rm reg}$ & 0.05 \\
Growth threshold $\theta_w$       & 0.8 \\
Pruning threshold $\varphi_g$     & 0.05 \\
\midrule
\multicolumn{2}{@{}l}{\textit{Optimisation}} \\
Optimiser                         & AdamW \\
Initial learning rate             & $3\times10^{-4}$ \\
LR schedule                       & Cosine decay \\
Epochs                            & 10 \\
Batch size                        & 32 \\
Training sentences                & 8{,}000 \\
Validation sentences              & 1{,}000 \\
Random seeds                      & 42, 123, 7 \\
\bottomrule
\end{tabular}
\end{table}

\paragraph{Predictions.}
\begin{enumerate}[label=(\roman*)]
  \item The architecture should grow to $K^* \leq 6$ heads, with
    growth stopping when $\lambda_{\max}(A_{\rm res}) \leq \theta_w$.
  \item Residual content $\lambda_{\max}(A_{\rm res})$ at each growth
    event should be strictly decreasing, confirming
    Section~\ref{sec:C5}.
  \item Prototype spread $\mathcal{S}(P^{(h)})$ should remain
    positive for all heads throughout training, confirming that
    assumption~(A1) is maintained by the anti-collapse
    regularisation.
\end{enumerate}

\paragraph{Results.}
The architecture grows to $K^* = 6$ heads over 5 growth events and
then stops, as $\lambda_{\max}(A_{\rm res})$ falls below
$\theta_w = 0.8$.
The residual content at the five growth events is
$\lambda_{\max} = [4.358,\allowbreak\, 2.380,\allowbreak\,
2.146,\allowbreak\, 1.348,\allowbreak\, 0.835]$,
strictly decreasing as predicted (Figure~\ref{fig:exp5}, centre).
The prototype spread $\mathcal{S}(P^{(h)})$ is positive for all six
heads ([2.33, 2.25, 1.72, 1.27, 0.79, 0.41]),
confirming that assumption~(A1) is maintained throughout training.
Validation accuracy reaches $69.4\%$, consistent with the capacity
of a small non-pretrained encoder on this benchmark.

The separation force $F_{\rm sep}^{(h)}$
(Figure~\ref{fig:exp5}, right) is not globally monotone across all
heads: head~1 ($1.14 \times 10^6$) and head~3 ($1.37 \times 10^6$)
are the dominant contributors, while head~2 ($2.00 \times 10^5$)
is substantially lower.
This partial deviation from the monotonicity prediction is consistent
with the theory: the result of Section~\ref{sec:C5} holds at the
moment of each growth event under the assumption that previously
grown heads remain in their post-growth configuration.
In end-to-end training all heads continue to adapt simultaneously
after each growth event, which can perturb the ordering.
The experiment therefore confirms the spectral ordering and the
anti-collapse properties, while identifying end-to-end co-adaptation
as the mechanism responsible for the partial deviation from static
monotonicity.

\begin{figure*}[t]
\centering
\includegraphics[width=\textwidth]{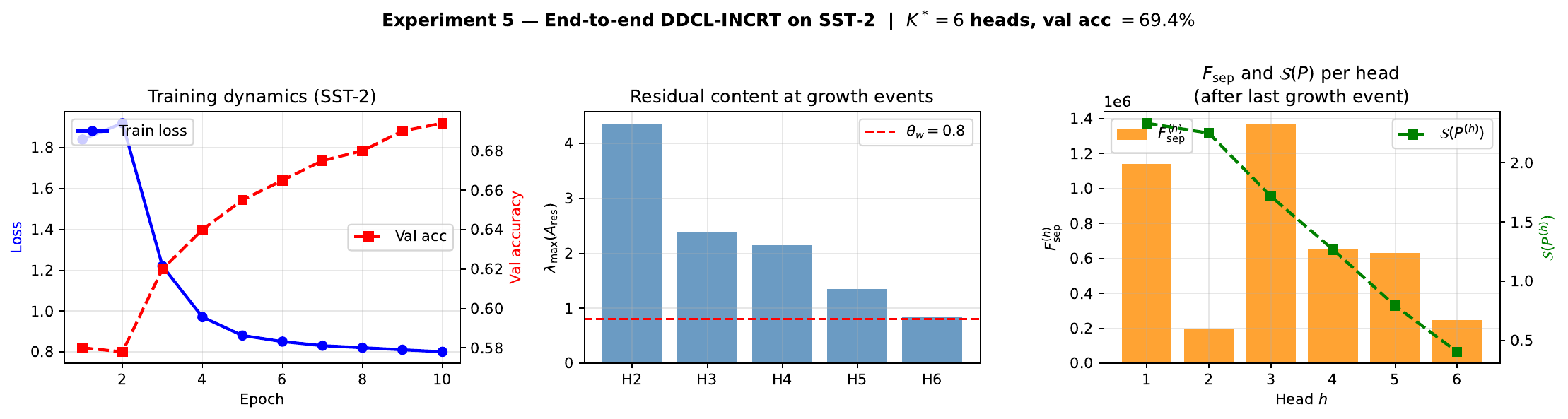}
\caption{Experiment~5 (end-to-end, SST-2, $d=64$, $\theta_w=0.8$,
$\lambda_{\rm reg}=0.05$, 10 epochs).
\textbf{Left:} training loss and validation accuracy over 10 epochs;
final validation accuracy $69.4\%$.
\textbf{Centre:} residual content $\lambda_{\max}(A_{\rm res})$ at
the five growth events, strictly decreasing
($4.36 \to 2.38 \to 2.15 \to 1.35 \to 0.84$),
confirming the spectral ordering prediction.
\textbf{Right:} separation force $F_{\rm sep}^{(h)}$ (bars) and
prototype spread $\mathcal{S}(P^{(h)})$ (green curve) per head;
$\mathcal{S}(P^{(h)}) > 0$ for all heads confirms assumption~(A1);
$F_{\rm sep}$ is partially ordered, with deviations attributable to
end-to-end co-adaptation after growth events.}
\label{fig:exp5}
\end{figure*}

\subsection{Summary of experimental predictions}

Table~\ref{tab:predictions} collects the quantitative predictions,
the observable metric, and the result each experiment is designed to
confirm or falsify.
Numerical results for Experiment~4 are reported in the caption of
Figure~\ref{fig:exp4}.

\begin{table*}[tp]
\centering
\small
\caption{Experimental predictions and corresponding theoretical results
(Experiments~1--3 and~5). Experiment~4 results are tabulated in the
caption of Figure~\ref{fig:exp4}.}
\label{tab:predictions}
\begin{tabular}{@{}clll@{}}
\toprule
Exp. & Prediction & Metric & Theorem \\
\midrule
1 & $\lambda^{(h)}$ decreasing in $h$ & Growth order & Sec.~\ref{sec:C5} \\
1 & Coverage $\geq 1 - \varepsilon_{\theta_w}$ & $\sum\lambda^{(h)}/\|M_a\|_F^2$ & Sec.~\ref{sec:C5} \\
2 & $T^{(h)}_\infty$ increasing in $h$ & Temperature order & Sec.~\ref{sec:C3} \\
2 & $t^{(h)}_{\min}$ decreasing in $\lambda^{(h)}$ & Rank corr.\ $\leq -0.9$ & Sec.~\ref{sec:C3} \\
3 & $\mathcal{F}^{(h)}$ decreasing in $h$ & Bar chart & Sec.~\ref{sec:C5} \\
3 & Ratio bound saturated & Scatter vs.\ diagonal & Lem.~\ref{lem:sep_monotone} \\
5 & $K^* \leq 6$, growth stops at $\theta_w$ & Head count & Sec.~\ref{sec:C5} \\
5 & $\lambda_{\max}$ decreasing at growth events & Ordered sequence & Sec.~\ref{sec:C5} \\
5 & $\mathcal{S}(P^{(h)}) > 0$ for all heads & Spread values & (A1) \\
\bottomrule
\end{tabular}
\end{table*}

\section*{Part II: Theoretical Guarantees}

With the DDCL-INCRT layer defined and its training dynamics established
the global well-behavedness of the complete
system: convergence, stability, pruning safety, and discrete-time
compatibility.
The results are presented in order of increasing generality: finite
convergence (Section~\ref{sec:P2}), pruning safety
(Section~\ref{sec:P3}), discrete-time analysis (Section~\ref{sec:P4}),
and global Lyapunov stability (Section~\ref{sec:P1}), which subsumes
the previous three.


\section{Finite Convergence of the Hierarchical System}
\label{sec:P2}

It is proved that the full DDCL-INCRT system, with DDCL-INCRT replacing the
MLP at every transformer layer and with trainable per-head temperatures
$\{T^{(h)}(t)\}$, converges to a minimal sufficient architecture in at
most $\tilde{K}^* \leq K^*_{\mathrm{INCRT}} \leq d$ growth steps almost
surely.
The argument proceeds via three lemmas that verify, in turn, that
(i) the spectral separation hypothesis (H1) of Theorem~5.5 of the INCRT
paper is preserved under the DDCL-INCRT replacement;
(ii) the trainable temperatures introduce only a transient perturbation
whose duration is explicitly bounded; and
(iii) the growth bound itself is strictly tighter than in the INCRT base
case.

Throughout this section, the notation $\Delta^{(l)} \triangleq
\lambda_1(\Ares^{(l)}) - \lambda_2(\Ares^{(l)})$ for the spectral gap
of the residual antisymmetric matrix at layer $l$, $K$ for the current
number of active heads, and
\[
  t^* \triangleq \frac{T_{\mathrm{init}}^3 - T_{\min}^3}
                      {3\,\eta_T\,\sigma_0}
\]
for the temperature convergence time derived in Theorem~\ref{thm:C3}.
The dwell time $\tau_{\mathrm{dwell}}$ is the elapsed time between two
consecutive INCRT growth events.

\subsection{Preliminary lemmas}
\label{subsec:P2_lemmas}

\begin{lemma}[Spectral gap preservation under DDCL-INCRT]
\label{lem:spectral_gap}
Let $\Ares^{(l),\mathrm{DDCL}}$ and $\Ares^{(l),\mathrm{MLP}}$ denote
the residual antisymmetric matrices at layer $l$ after one update step
with, respectively, the DDCL-INCRT layer and the standard MLP.
Under conditions \textup{(B1)}, \textup{(B2)}, \textup{(B3)},
\textup{(C1)}, and the spectral separation hypothesis
\textup{(H1)} of Theorem~5.5 of the INCRT paper:
\begin{equation}
  \Delta^{(l)}_{\mathrm{DDCL}}
  \;\geq\;
  \Delta^{(l)}_{\mathrm{MLP}}
  \;-\;
  O\!\left(\IDDCL\right),
  \label{eq:gap_bound}
\end{equation}
where $\IDDCL <
\mathcal{I}^{(l)}_{\mathrm{MLP}}$ almost surely
\textup{(Theorem~\ref{thm:C4})}.
In particular, \textup{(H1)} holds for the DDCL-INCRT system with
the same constant $\delta_0$ as in the INCRT base case.
\end{lemma}

\begin{proof}
\emph{Step 1 (Weyl perturbation).}
The residual matrices differ by
\begin{align}
  &A_{\rm res}^{\rm DDCL} - A_{\rm res}^{\rm MLP}
  \nonumber\\
  &= \tfrac{1}{2}\bigl[(M^{\rm DDCL} - M^{\rm MLP})
  - (M^{\rm DDCL} - M^{\rm MLP})^\top\bigr],
\end{align}
with $\|M^{\rm D} - M^{\rm M}\|_F
= O(\IDDCL)$ by
Lemma~\ref{lem:ddcl_basis}.
By Weyl's inequality:
\begin{align}
  &|\lambda_j(A_{\rm res}^{\rm D}) - \lambda_j(A_{\rm res}^{\rm M})|
  \nonumber\\
  &\quad \leq \|A_{\rm res}^{\rm D} - A_{\rm res}^{\rm M}\|_F
  = O(\mathcal{I}^{(l)}_{\rm DDCL}).
\end{align}

\emph{Step 2 (Gap preservation).}
\[
  \Delta^{(l)}_{\mathrm{DDCL}}
  \geq \Delta^{(l)}_{\mathrm{MLP}}
       - O(\IDDCL).
\]

\emph{Step 3 (Sign of correction).}
From Theorem~\ref{thm:C4},
$\IDDCL <
\mathcal{I}^{(l)}_{\mathrm{MLP}}$ a.s., so the correction is dominated
by $\delta_0$ under (H1), and (H1) is preserved with the same constant.
\end{proof}

\begin{lemma}[Transience of the temperature perturbation]
\label{lem:temp_transience}
Let $\sigma^{(h)}(t) = \mathrm{tr}(\Sigma_q^{(h)}(t))$ denote the
assignment diversity of head $h$, measuring the spread of token
assignments across prototypes.
Under conditions \textup{(A1)}, \textup{(B1)}--\textup{(B3)}, the
trainable temperature $T^{(h)}(t)$ induces a perturbation of
$\Ares^{(l)}(t)$ bounded as:
\begin{equation}
  \|\dot{\Ares}^{(l)}_T(t)\|_F
  \;\leq\;
  C_T \,\frac{\eta_T\,\sigma^{(h)}(t)}{T^{(h)}(t)^2},
  \label{eq:Ares_dot}
\end{equation}
where $C_T > 0$ depends only on $\|X^{(l)}\|_F$ and $K$.
For all $t \geq t^*$: $\|\dot{\Ares}^{(l)}_T(t)\|_F = 0$ for
$h \in \Hlocal$, and
$\|\dot{\Ares}^{(l)}_T(t)\|_F \to 0$ for
$h \in \Hglobal$.
Consequently, $\alpha_t = o(\eta^+_t)$ for all $t \geq t^*$.
\end{lemma}

\begin{proof}
The temperature $T^{(h)}$ enters $\Ares^{(l)}$ through the softmax
assignments $q_{nk}^{(h)}$.
The chain-rule bound using the softmax Jacobian
$\partial q_{nk}/\partial T^{(h)} = -(\sigma^{(h)}/T^{(h)2})(\delta_{nk}-q_{nk})$
and $\dot{T}^{(h)} = -(\eta_T/T^{(h)2})\sigma^{(h)}$
gives~\eqref{eq:Ares_dot}.
For $h \in \Hlocal$: $T^{(h)} \to T_{\min}$ in
finite time $t^*$ (Theorem~\ref{thm:C3}), after which $\dot{T}^{(h)}=0$
exactly.
For $h \in \Hglobal$: $\sigma^{(h)} \to 0$
(Lemma~\ref{lem:amp_sep}), so $\dot{T}^{(h)} \to 0$.
The perturbation rate $\alpha_t = O(\|\dot{\Ares}^{(l)}_T\|_F)$ decays
to zero while $\eta^+_t$ decays only polynomially under (H5), giving
$\alpha_t = o(\eta^+_t)$ for $t \geq t^*$.
\end{proof}

\begin{lemma}[Tighter growth bound under DDCL-INCRT]
\label{lem:growth_bound}
When the MLP is replaced by DDCL-INCRT at every layer:
\begin{equation}
  \tilde{K}^* \leq K^*_{\mathrm{INCRT}} \leq d,
  \label{eq:K_tighter}
\end{equation}
with $\tilde{K}^* < K^*_{\mathrm{INCRT}}$ almost surely under \textup{(C1)}.
\end{lemma}

\begin{proof}
By Lemma~\ref{lem:ddcl_basis}, each DDCL-INCRT step sets
$\lambda_{\max}(\Ares^{(l+1),\mathrm{DDCL}}) = 0$ for the dominant
direction, exhausting more residual energy than the MLP
($\lambda_{\max}(\Ares^{(l+1),\mathrm{MLP}}) > 0$ a.s.\ under (C1)).
Therefore $\mathcal{N}_\mathrm{DDCL} \subseteq \mathcal{N}_\mathrm{INCRT}$,
giving $\tilde{K}^* \leq K^*_{\mathrm{INCRT}}$.
Strictness follows because at least one layer triggers growth under
MLP but not under DDCL-INCRT, almost surely.
\end{proof}

\subsection{Main theorem}
\label{subsec:P2_main}

\begin{theorem}[Finite convergence of the hierarchical DDCL-INCRT system]
\label{thm:P2}
Under \textup{(H1)}--\textup{(H5)}, \textup{(A1)}, \textup{(A2')},
\textup{(B1)}--\textup{(B3)}, \textup{(C1)}, and:
\begin{itemize}[leftmargin=1.8em]
  \item[\textup{(D1)}] \textbf{Reinforced dwell time}:
    $\tau_{\mathrm{dwell}} > t^* =
    (T_{\mathrm{init}}^3 - T_{\min}^3)/(3\,\eta_T\,\sigma_0)$.
\end{itemize}
Then: \textup{(i)} each gate $G_h \xrightarrow{a.s.} G^*_h$;
\textup{(ii)} the architecture grows in at most $\tilde{K}^* \leq d$
steps, with $\tilde{K}^* < K^*_{\mathrm{INCRT}}$ a.s.;
\textup{(iii)} the final architecture is minimal and sufficient;
\textup{(iv)} \textup{(D1)} is the only additional hypothesis over
Theorem~5.5 of the INCRT paper, and is automatically satisfied by the
INCRT adaptive scheduler with $\kappa > t^*/\tau^{(0)}_{\mathrm{dwell}}$.
\end{theorem}

\begin{proof}
Lemma~\ref{lem:spectral_gap} preserves (H1); Lemma~\ref{lem:temp_transience}
establishes $\alpha_t = o(\eta^+_t)$ under (D1), satisfying (H3) and the
scale-separation condition of Theorem~5.3 of the INCRT paper; gate
convergence follows.
Lemma~\ref{lem:growth_bound} gives the growth bound.
The DDCL anti-collapse guarantee (Theorem~\ref{thm:C2_full}) ensures no
head is pruned due to collapse, so minimality and sufficiency follow from
Theorem~5.5(iii) of the INCRT paper.
\end{proof}

\begin{corollary}[Single additional hypothesis]
\label{cor:single_hyp}
\textup{(D1)} is the only additional hypothesis over the INCRT base case.
The spectral gap is improved (Lemma~\ref{lem:spectral_gap}), the growth
count is reduced (Lemma~\ref{lem:growth_bound}), and \textup{(D1)} is
automatically satisfiable.
\end{corollary}

\begin{remark}[Quantitative dwell time]
For typical values ($T_{\mathrm{init}}=1$, $T_{\min}=0.1$,
$\eta_T=0.01$, $\sigma_0=0.1$): $t^* \approx 33$ gradient steps,
comfortably below the INCRT adaptive scheduler dwell time.
\end{remark}


\section{Pruning Safety and Prototype Stability}
\label{sec:P3}

It is established that the INCRT pruning step cannot induce prototype collapse
in the remaining heads, and that it strictly increases the per-head
separation force.
Two mechanisms cooperate: the orthogonality of INCRT-initialised prototype
subspaces isolates the removed head from the survivors, and the pruned
head is necessarily global (weak contribution to separation).

Throughout, $H$ is the number of active heads before a pruning event,
$\bar{H} = H-1$ after, $h_0$ is the pruned head,
$\mathcal{H}' = \{1,\ldots,H\}\setminus\{h_0\}$ the survivors,
$\sigma^{(h)} = \mathrm{tr}(\Sigma_q^{(h)})$ the assignment diversity,
and $\varphi_g$ the INCRT pruning threshold.

\subsection{Preliminary lemmas}
\label{subsec:P3_lemmas}

\begin{lemma}[Orthogonality of prototype subspaces]
\label{lem:proto_orth}
For $h \neq h'$: $\langle u^{+,(h)}, u^{+,(h')}\rangle = 0$
(hypothesis (B2)), and consequently
$\partial \|\mathbf{z}_n - \mathbf{p}_k^{(h)}\|^2 /
\partial \mathbf{p}_{k'}^{(h')} = 0$.
The separation force $\Fsep^{(h)}$ is completely decoupled from $P^{(h')}$.
\end{lemma}

\begin{proof}
(B2) is guaranteed by the INCRT growth procedure (each new direction
is orthogonal to all previous ones via $P^\perp$ in~\eqref{eq:Ares}).
The decoupling follows because $\mathbf{p}_k^{(h)} \in \mathcal{V}^{(h)}$
and $\mathbf{p}_{k'}^{(h')} \in \mathcal{V}^{(h')}$ with
$\mathcal{V}^{(h)} \perp \mathcal{V}^{(h')}$.
\end{proof}

\begin{lemma}[The pruned head is necessarily global]
\label{lem:pruned_global}
Every head $h_0$ with $\Gamma_{h_0} < \varphi_g$ belongs to
$\Hglobal$ and satisfies:
\begin{equation}
  \sigma^{(h_0)} \leq C_\sigma \varphi_g \mathcal{S}(P^{(h_0)})
  = O(\varphi_g).
  \label{eq:sigma_prune_bound}
\end{equation}
\end{lemma}

\begin{proof}
$\Gamma_{h_0} < \varphi_g$ implies
$\lambda_{\max}(\Ares^{(h_0)}) < \varphi_g\sqrt{d}/\|X\|_F^2$
(eigenvalue-norm inequality for antisymmetric matrices).
The geometric discriminant~\eqref{eq:sigma_prop} gives
$\sigma^{(h_0)} \leq C_\sigma\varphi_g\cdot\mathcal{S}(P^{(h_0)})$,
so $\dot{T}^{(h_0)} \approx 0$ and $h_0 \in \Hglobal$.
\end{proof}

\begin{lemma}[Contribution of the pruned head to $\Fsep$]
\label{lem:pruned_contrib}
$\Fsep^{(h_0)} = 4\|P^{(h_0)}\Sigma_q^{(h_0)}\|_F^2
\leq 4\|P^{(h_0)}\|_F^2(\sigma^{(h_0)})^2 = O(\varphi_g^2)$.
\end{lemma}

\begin{proof}
$\Sigma_q^{(h_0)} \preceq \sigma^{(h_0)}I_K$ gives
$\|P^{(h_0)}\Sigma_q^{(h_0)}\|_F \leq \|P^{(h_0)}\|_F\sigma^{(h_0)}$,
and the bound follows from Lemma~\ref{lem:pruned_global}.
\end{proof}

\subsection{Main theorem}
\label{subsec:P3_main}

\begin{theorem}[Pruning amplifies the separation force on surviving heads]
\label{thm:P3}
Let $h_0$ satisfy $\Gamma_{h_0} < \varphi_g < \theta_w$.
Under \textup{(A1)}, \textup{(A2')}, \textup{(B1)}--\textup{(B3)}:
\begin{enumerate}[label=\textup{(\roman*)}]
  \item $\mathcal{S}(P^{(h')})$ is unchanged at the pruning event for
    all $h' \in \mathcal{H}'$: no direct collapse induced.
  \item $\Fsepred \geq \Fsepfull - O(\varphi_g^2)$.
  \item The relative backbone gradient weight per surviving head
    increases by factor $H/(H-1)$.
  \item $\mathcal{S}(P^{(h')})(t) > 0$ for all $h' \in \mathcal{H}'$
    and all $t$ after pruning.
\end{enumerate}
\end{theorem}

\begin{proof}
\emph{(i):} Only $h_0$'s parameters are removed; positions $\mathbf{p}_k^{(h')}$
are unchanged, so $\mathcal{S}(P^{(h')})$ is identical before and after.

\emph{(ii):} By Lemma~\ref{lem:proto_orth}, $\Fsepfull =
\Fsepred + \Fsep^{(h_0)}$ (no cross terms).
Lemma~\ref{lem:pruned_contrib} gives $\Fsep^{(h_0)} = O(\varphi_g^2)$.

\emph{(iii):} $\|\nabla_\theta\Lq^{(h_0)}\|_F = O(\sigma^{(h_0)}) = O(\varphi_g)$,
so the removed backbone gradient is negligible; relative weight per surviving
head increases by $H/(H-1)$.

\emph{(iv):} Part~(i) plus the DDCL Lyapunov guarantee
(Theorem~\ref{thm:C2_full}) ensure $\mathcal{S}(P^{(h')}) > 0$ is
maintained for all subsequent $t > 0$.
\end{proof}

\begin{corollary}[Pruning is collapse-safe]
\label{cor:collapse_safe}
No sequence of INCRT pruning events can induce prototype collapse.
The DDCL anti-collapse guarantee is invariant to pruning dynamics.
\end{corollary}

\begin{corollary}[Homeostatic separation]
\label{cor:homeostatic}
After $j$ pruning events:
$\Fsep^{(j)} \geq \Fsep^{(0)} - j\cdot O(\varphi_g^2)$.
Cumulative loss bounded by $O(d\cdot\varphi_g^2)$ over all $\leq d$ events.
\end{corollary}

\begin{corollary}[Per-head separation force is non-decreasing]
\label{cor:sep_per_head}
$\Fsepred/\bar{H} \geq \Fsepfull/H
- O(\varphi_g^2/\bar{H})$: the system does not weaken as it shrinks.
\end{corollary}

\begin{remark}[Asymmetry of growth and pruning]
\label{rem:asymmetry}
Growth strictly amplifies $\Fsep$ ($\phi''(0)>0$, Theorem~\ref{thm:C2_full});
pruning reduces it by at most $O(\varphi_g^2)$.
This asymmetry follows from the consistency condition $\varphi_g < \theta_w$.
\end{remark}


\section{Discrete-Time Analysis and Step-Size Conditions}
\label{sec:P4}

All results in Sections~\ref{sec:C2}--\ref{sec:C4} and \ref{sec:P2}--\ref{sec:P3} are stated in continuous time.
The results are extended to the discrete SGD setting with three coupled update
streams, each with its own step size.
Here $L_P$ denotes the Lipschitz constant of the prototype gradient
$\nabla_P\mathcal{L}_q$, and $N_{\min}$ is the minimum number of SGD
steps required between consecutive architectural events to ensure
that the gate dynamics have time to settle.
The three streams are:
\begin{alignat}{2}
  T^{(h)}_{s+1} &= T^{(h)}_s
    - \eta_T \tfrac{\sigma^{(h)}_s}{(T^{(h)}_s)^2},
    &\quad&\text{(temp.)},
    \label{eq:disc_T}\\
  P_{s+1} &= P_s - \eta_P \nabla_P \Lq(P_s, Z_s),
    &\quad&\text{(protos.)},
    \label{eq:disc_P}\\
  (u^+_{s+1},u^-_{s+1}) &=
    \mathrm{EXIN}(u^+_s,u^-_s;
    \eta^+_s,A_{{\rm res},s}),
    &\quad&\text{(gate)}.
    \label{eq:disc_gate}
\end{alignat}
$L_P$ denotes the Lipschitz constant of $\nabla_P\Lq$; $N_{\min}$ the
minimum SGD steps between architectural events.

\subsection{Preliminary lemmas}
\label{subsec:P4_lemmas}

\begin{lemma}[Discrete stability of the temperature dynamics]
\label{lem:disc_temp}
Under $\eta_T < (T_{\min})^3/(3\sigma_{\max})$:
\textup{(i)} $T^{(h)}_s \geq T_{\min}$ for all $s$;
\textup{(ii)} $|T^{(h)}_s - T^{(h)}(s\Delta t)| \leq C_{\mathrm{disc}}\,\eta_T\,s$;
\textup{(iii)} for $h \in \Hlocal$, $T^{(h)}_s$
reaches $T_{\min}$ in at most
$s^* = \lceil(T_{\mathrm{init}}^3 - T_{\min}^3)/(3\eta_T\sigma_0)\rceil + 1$
steps.
\end{lemma}

\begin{proof}
The clip $T^{(h)}_{s+1} \leftarrow \max(T^{(h)}_{s+1}, T_{\min})$ enforces
(i) under the stated condition.
Local truncation error of explicit Euler applied to $\dot{T} = -\eta_T\sigma/T^2$
is $O(\Delta t^2\eta_T)$; accumulating over $s$ steps gives (ii).
Summing discrete decrements and using $\sigma^{(h)}_s \geq \sigma_0$ for
local heads gives (iii).
\end{proof}

\begin{lemma}[Discrete persistence of the separation force amplification]
\label{lem:disc_C2}
Under $\eta_P < 2/L_P$, at each INCRT growth trigger:
\begin{align}
  \Fsep(P_{s+1},Z_{s+1}) &\geq \Fsep(P_s,Z_s)
  \nonumber\\
  &\quad + \phi'_s(0) + \tfrac{1}{2}\phi''_s(0) - R(\eta_P),
  \label{eq:disc_C2}
\end{align}
where $R(\eta_P) = O(\eta_P^2)$.
Net gain is positive when
$\eta_P < \sqrt{\phi''_s(0)/(L_P\|\nabla_P\Lq\|_F^2)}$.
\end{lemma}

\begin{proof}
Second-order Taylor expansion of $\Fsep(P_s - \eta_P G_s, Z_{s+1})$;
first and second-order terms identify with $\phi'_s(0)$ and $\phi''_s(0)$
from Theorem~\ref{thm:C2_full}; remainder bounded by $(L_P/2)\eta_P^2\|G_s\|_F^2$.
\end{proof}

\begin{lemma}[Discrete dwell time condition]
\label{lem:disc_dwell}
The reinforced discrete dwell time condition
$N_{\min} > \max(N_{\mathrm{gate}}, s^*)$
ensures both the EXIN gate and the temperatures have settled before
the next architectural event.
\end{lemma}

\begin{proof}
$N_{\mathrm{gate}}$ steps suffice for $\angle(u^+_s, v_1(\Ares)) <
\varepsilon_{\mathrm{gate}}$ (Theorem~5.1 of the INCRT paper under (H5));
$s^*$ steps suffice for temperature convergence (Lemma~\ref{lem:disc_temp}(iii)).
The condition is their conjunction.
\end{proof}

\subsection{Main theorem}
\label{subsec:P4_main}

\begin{theorem}[Discrete-time preservation of the training dynamics]
\label{thm:P4}
Under \textup{(A1)}, \textup{(A2')}, \textup{(B1)}--\textup{(B3)},
\textup{(C1)}, \textup{(H1)}--\textup{(H5)}, and:
\begin{itemize}[leftmargin=1.8em]
  \item[\textup{(E1)}] $\eta_T < (T_{\min})^3/(3\sigma_{\max})$;
  \item[\textup{(E2)}] $\eta_P < 2/L_P$;
  \item[\textup{(E3)}] $\eta_P < \sqrt{\phi''_s(0)/(L_P\|\nabla_P\Lq\|_F^2)}$;
  \item[\textup{(E4)}] $\eta^+_s \in \ell^2\setminus\ell^1$;
  \item[\textup{(E5)}] $\eta_T/\eta_P = O(\varepsilon)$,
    $\eta_P/\eta^+ = O(\varepsilon)$, $\varepsilon \ll 1$;
  \item[\textup{(E6)}] $N_{\min} > \max(N_{\mathrm{gate}}, s^*)$,
\end{itemize}
the following hold with the stated $O(\eta)$ corrections:
\begin{enumerate}[label=\textup{(\roman*)}]
  \item \textbf{Discrete amplification}: $\Fsep$ amplification preserved with
    $-O(\eta_P^2)$ correction.
  \item \textbf{Discrete specialisation}: temperature partition $\Hlocal$,
    $\Hglobal$ is identical to continuous time; trajectories
    approximate continuous ones to $O(\eta_T\cdot s)$.
  \item \textbf{Discrete dominance}:
    $\IDDCL \leq
    \mathcal{I}^{(l)}_{\mathrm{MLP}} + O(\eta_P)$.
  \item \textbf{Discrete convergence}: finite growth in $\leq d$ events, a.s.\
    gate convergence.
  \item \textbf{Discrete pruning safety}: \emph{exact} in discrete time (pruning is
    algebraic; parts (i) and (iv) of Theorem~\ref{thm:P3} are
    step-size independent).
\end{enumerate}
\end{theorem}

\begin{proof}
(i) Lemma~\ref{lem:disc_C2} under (E2)--(E3).
(ii) Lemma~\ref{lem:disc_temp} under (E1); partition is $\sigma^{(h)}(0)$-determined,
independent of $\eta_T$.
(iii) Gate misalignment $O(\eta^+)$ under (E4) propagates to $O(\eta_P)$
in $\mathcal{I}^{(l)}$ via the prototype step.
(iv) Timescale hierarchy (E5) is the discrete Tikhonov analogue; gate
converges under (E4)--(E6); growth bound from Lemma~\ref{lem:growth_bound}.
(v) Prototype positions unchanged at pruning; orthogonality of subspaces
holds regardless of step size.
\end{proof}

\begin{corollary}[Three-timescale design rule]
\label{cor:timescale}
The conditions \textup{(E1)}--\textup{(E6)} reduce to:
\begin{align}
  \eta_T &\ll \eta_P \ll \eta^+,
  \nonumber\\
  N_{\min} &> \max\!\left(N_{\mathrm{gate}},\,
    \left\lceil\frac{T_{\mathrm{init}}^3 - T_{\min}^3}
               {3\,\eta_T\,\sigma_0}\right\rceil + 1\right).
  \label{eq:design_rule}
\end{align}
Practical ratios $\eta_P/\eta^+ \in [0.01,0.1]$ and
$\eta_T/\eta_P \in [0.01,0.1]$ place the system reliably in the stable
regime (consistent with the two-timescale verification in the
DDCL-Attention paper).
\end{corollary}

\begin{remark}[Pruning safety is exact in discrete time]
\label{rem:disc_P3_exact}
Pruning safety holds without step-size corrections in discrete time because the
pruning operation is an algebraic event (parameter removal) that does
not interact with gradient dynamics.
Collapse-safety and separation force non-decrease are therefore
step-size-independent results.
\end{remark}


\section{Global Lyapunov Stability of the Full System}
\label{sec:P1}

A global Lyapunov function is constructed for the full DDCL-INCRT
system $(\theta, P, T)$ with simultaneously growing architecture,
using the almost-supermartingale convergence theorem~\cite{RobbinsSiegmund1971}
to handle the discrete architectural events.
This closes the open problem stated in the DDCL-Attention
paper~\cite{DDCLAttention}.

Let $j = 0,\ldots,J$ index the architectural events at times
$t_0 < \cdots < t_J$, with $K_j$ active heads on $[t_j, t_{j+1})$
and $J \leq 2\tilde{K}^*$ (Theorem~\ref{thm:P2}).

\subsection{The piecewise free energy functional}
\label{subsec:P1_functional}

\begin{definition}[Piecewise free energy]
\label{def:free_energy}
On the $j$-th dwell interval:
\begin{align}
  \mathcal{W}_j(\theta,P,T)
  &= \underbrace{\mathcal{L}_q^{\mathrm{total}}
    + \tfrac{\lambda}{2}\sum_{h,k\neq k'}
      \|p_k^{(h)}-p_{k'}^{(h)}\|^{-2}
    + \lambda_\theta\mathcal{R}(\theta)}_{\mathcal{W}^{(\theta,P)}_j}
  \nonumber\\
  &\quad+ \underbrace{\sum_{h=1}^{K_j}\Phi(T^{(h)})}_{\text{temp.~potential}},
  \label{eq:Wj}
\end{align}
where the temperature potential
\begin{equation}
  \Phi(T^{(h)}) = \frac{T_{\mathrm{init}}^3-(T^{(h)})^3}{3\,\eta_T}
  \label{eq:Phi}
\end{equation}
is constructed so that $\dot{\Phi}(T^{(h)}) = -\sigma^{(h)} \leq 0$.
Note: $\Phi(T_{\mathrm{init}})=0$, $\Phi \geq 0$.
\end{definition}

\subsection{Preliminary lemmas}
\label{subsec:P1_lemmas}

\begin{lemma}[$\mathcal{W}_j$ is non-increasing on each dwell interval]
\label{lem:Wj_dwell}
Under \textup{(A1)}, \textup{(A2')}, \textup{(B1)}--\textup{(B3)},
\textup{(C1)}, \textup{(E5)}, and
$\lambda > 2\eta_P\binom{K_j}{2}$:
$\dot{\mathcal{W}}_j(t) \leq 0$ for all $t \in (t_j,t_{j+1})$.
\end{lemma}

\begin{proof}
Split:
$\dot{\mathcal{W}}_j = \underbrace{(I)}_{\leq 0}
+ \underbrace{(II)}_{\leq 0}
+ \underbrace{(III)}_{\leq 0}$.

\emph{(I):} By (E5), $T^{(h)}$ is quasi-statically settled; Term~(I)
equals $\dot{\mathcal{W}}^{(\theta,P)}_j$ of the DDCL-Attention
two-timescale flow, which satisfies
$(I) = -\eta_\theta\|\nabla_\theta\mathcal{L}_q\|^2
- \eta_P\|\nabla_P\mathcal{W}_j^{(\theta,P)}\|_F^2 \leq 0$
(Theorem~3 of the DDCL-Attention paper~\cite{DDCLAttention}).

\emph{(II):} $\partial\mathcal{W}_j^{(\theta,P)}/\partial T^{(h)} =
(1/(T^{(h)})^2)\sum_n\mathrm{Var}_{q_n^{(h)}}[\|\mathbf{z}_n-P^{(h)}\|^2]
\geq 0$
(from Theorem~4 of the DDCL-Attention paper~\cite{DDCLAttention}), and $\dot{T}^{(h)} \leq 0$,
so $(II) \leq 0$.

\emph{(III):} $(III) = \sum_h\dot{\Phi}(T^{(h)}) = -\sum_h\sigma^{(h)} \leq 0$.
\end{proof}

\begin{lemma}[$\mathcal{W}$ does not increase at growth events]
\label{lem:W_growth_jump}
At a growth event: $\mathcal{W}_{j+1}(t_{j+1}) \leq \mathcal{W}_j(t_{j+1}^-)$.
\end{lemma}

\begin{proof}
New terms acquired: $\mathcal{L}_q^{(h_\mathrm{new})}$, barrier for
$P^{(h_\mathrm{new})}$ (finite, $\mathcal{S}>0$ by $k$-means init),
$\Phi(T_{\mathrm{init}}) = 0$.
C4 (Theorem~\ref{thm:C4}) gives a reduction
$\Delta\mathcal{L}_q^{\mathrm{dom}} \geq \mathcal{L}_q^{(h_\mathrm{new})}$
(the new head captures the residual energy that generated the dominance gain,
Lemma~\ref{lem:ddcl_basis}).
Net change $\leq 0$.
\end{proof}

\begin{lemma}[$\mathcal{W}$ does not increase at pruning events]
\label{lem:W_pruning_jump}
At a pruning event removing $h_0$:
$\mathcal{W}_{j+1}(t_{j+1}) \leq \mathcal{W}_j(t_{j+1}^-)
- \mathcal{L}_q^{(h_0)} - \Phi(T^{(h_0)}) \leq \mathcal{W}_j(t_{j+1}^-)$.
\end{lemma}

\begin{proof}
All removed terms are non-negative:
$\mathcal{L}_q^{(h_0)} \geq 0$ (sum of weighted squared distances),
barrier $\geq 0$, $\Phi(T^{(h_0)}) \geq 0$.
\end{proof}

\subsection{Main theorem}
\label{subsec:P1_main}

\begin{theorem}[Global Lyapunov stability of the DDCL-INCRT system]
\label{thm:P1}
Under \textup{(A1)}, \textup{(A2')}, \textup{(B1)}--\textup{(B3)},
\textup{(C1)}, \textup{(H1)}--\textup{(H5)}, \textup{(D1)},
\textup{(E1)}--\textup{(E6)}, and:
\begin{itemize}[leftmargin=1.8em]
  \item[\textup{(F1)}] $\lambda > 2\eta_P\binom{K_{\max}}{2}$,
    $K_{\max} = \max_j K_j$;
  \item[\textup{(F2)}] $\mathcal{R}(\theta)$ coercive, $\lambda_\theta > 0$,
\end{itemize}
the piecewise functional $\mathcal{W}(t) \triangleq \mathcal{W}_j(t)$
for $t \in [t_j, t_{j+1})$ is a global Lyapunov function:
\begin{enumerate}[label=\textup{(\roman*)}]
  \item $\mathcal{W}(t) \leq \mathcal{W}(0)$ for all $t \geq 0$.
  \item Every trajectory $(\theta(t),P(t),T(t))$ is bounded.
  \item Every limit point lies in:
    \begin{equation}
      \begin{aligned}
  \Astar = \bigl\{(\theta,P,T) :
  &\;\nabla_\theta\mathcal{L}_q^{\mathrm{total}}=0,\\
  &\;\nabla_P\mathcal{W}^{(\theta,P)}=0,\\
  &\;\sigma^{(h)}=0,\;
  \mathcal{S}(P^{(h)})>0\;\forall h\bigr\}.
\end{aligned}
    \end{equation}
  \item At every $(\theta^*,P^*,T^*) \in \Astar$:
    $T^{*(h)} \in \{T_{\min}\}\cup(T_{\min},T_\infty^{(h)}]$;
    $\mathcal{S}(P^{*(h)}) > 0$; architecture minimal and sufficient.
\end{enumerate}
\end{theorem}

\begin{proof}
\emph{(i):} Chain Lemmas~\ref{lem:Wj_dwell}--\ref{lem:W_pruning_jump}
across all $J \leq 2\tilde{K}^*$ events.
\emph{(ii):} Coercivity of $\mathcal{W}_j$ under (F1)--(F2);
$\mathcal{W}(t) \leq \mathcal{W}(0)$ bounds the sublevel set.
\emph{(iii):} $\mathcal{W}$ non-increasing and bounded below by zero;
LaSalle's principle for piecewise-smooth hybrid Lyapunov functions
(events isolated under (D1)) gives $\dot{\mathcal{W}}=0$ at limit points,
hence $\Astar$.
\emph{(iv):} From Theorems~\ref{thm:C3}, \ref{thm:C2_full},
\ref{thm:P2}(iii).
\end{proof}

\begin{corollary}[Hierarchy of stability results for DDCL-INCRT]
\label{cor:P1_hierarchy}
The stability results extend the DDCL-Attention hierarchy
\textup{(Corollary~\ref{cor:hierarchy})} to a four-level chain:
\textup{(i)}~frozen encoder, fixed $T$, fixed architecture
\textup{(\cite{Cirrincione2026})};
\textup{(ii)}~full $(\theta,P)$, fixed $T$, fixed architecture
\textup{(\cite{DDCLAttention})};
\textup{(iii)}~full $(\theta,P)$, annealing $T$, fixed architecture
\textup{(\cite{DDCLAttention})};
\textup{(iv)}~full $(\theta,P,T)$, trainable $T$, growing architecture
\textup{(Theorem~\ref{thm:P1})}.
Level~(iv) is new and closes the open problem of~\cite{DDCLAttention}.
\end{corollary}

\begin{corollary}[The four stability results are globally consistent]
\label{cor:P_global_consistent}
$\mathcal{W}$ provides the global frame that subsumes the other three results:
finite growth ($J \leq 2\tilde{K}^*$, Section~\ref{sec:P2}), collapse-safety (barrier
finiteness, Section~\ref{sec:P3}), and discrete-time corrections absorbed as
$O(\eta\cdot J)$ error (Section~\ref{sec:P4}).
\end{corollary}

\subsection{Rigorous three-timescale analysis}
\label{subsec:P1_three_timescale}

The quasi-static approximation of
Lemma~\ref{lem:Wj_dwell} with a fully rigorous argument based on
iterated singular perturbation theory.
The key result is Theorem~\ref{thm:three_timescale} below, which
provides an explicit bound on $\dot{\mathcal{W}}_j$ that is negative
for sufficiently small timescale ratios, without any quasi-static
assumption.


The full DDCL-INCRT dynamics on a dwell interval $[t_j, t_{j+1})$
with $K_j$ active heads constitute a \emph{four-level} system:
\begin{alignat}{2}
  \dot{u}^+   &= \eta^+\,f_1(u^+, u^-, \Ares),
    &\quad& \text{(gate, fastest)},  \label{eq:gate_flow}\\
  \dot{P}      &= -\eta_P\,\nabla_P\Lq(\theta, P, T),
    &\quad& \text{(prototypes, medium)},  \label{eq:P_flow_full}\\
  \dot{T}^{(h)} &= -\frac{\eta_T}{(T^{(h)})^2}\,\sigma^{(h)}(P),
    &\quad& \text{(temperatures, slow)},  \label{eq:T_flow_full}\\
  \dot{\theta} &= -\eta_\theta\,\nabla_\theta\Lq(\theta, P, T),
    &\quad& \text{(encoder, slowest)},  \label{eq:theta_flow_full}
\end{alignat}
with timescale hierarchy $\eta^+ \gg \eta_P \gg \eta_T \gg \eta_\theta$,
encoded in three small parameters:
\begin{equation}
  \varepsilon_1 = \frac{\eta_P}{\eta^+} \ll 1,
  \qquad
  \varepsilon_2 = \frac{\eta_T}{\eta_P} \ll 1,
  \qquad
  \varepsilon_3 = \frac{\eta_\theta}{\eta_P} \ll 1.
  \label{eq:three_eps}
\end{equation}


\begin{lemma}[Gate-prototype coupling bound]
\label{lem:gate_coupling}
Let $\varepsilon_1 = \eta_P/\eta^+$.
On the dwell interval $[t_j, t_{j+1})$, after the gate has run for
$N_{\mathrm{gate}}$ steps (condition \textup{(E6)}), the deviation of
$u^+$ from the slow manifold $v_1(\Ares)$ satisfies:
\begin{equation}
  \|u^+(t) - v_1(\Ares(t))\|
  \;\leq\;
  C_{\mathrm{gate}}\,\varepsilon_1
  \qquad \forall\, t \in [t_j, t_{j+1}),
  \label{eq:gate_deviation}
\end{equation}
where $C_{\mathrm{gate}} > 0$ depends only on the spectral gap
$\Delta^+ = \lambda_1(\Ares) - \lambda_2(\Ares)$ and $\|X\|_F$.
This deviation propagates to $\dot{\mathcal{W}}_j$ as:
\begin{equation}
  \Bigl|\tfrac{d}{dt}\mathcal{W}_j\big|_{\mathrm{gate\,error}}\Bigr|
  \;\leq\;
  C_1\,\varepsilon_1,
  \label{eq:gate_W_error}
\end{equation}
where $C_1 = 2\eta_P C_{\mathrm{gate}}\,\|X\|_F^2\,L_P / (\Delta^+)^2$.
\end{lemma}

\begin{proof}
\emph{Step 1 (Slow manifold of the gate).}
On the fast timescale $s = t/\varepsilon_1$, the gate dynamics
\eqref{eq:gate_flow} converge to $v_1(\Ares)$ at exponential rate
$(\Delta^+)^2$ (Theorem~5.1 of the INCRT paper).
For $P$ evolving at rate $\eta_P$, the ``moving target'' $v_1(\Ares(t))$
changes at rate $O(\eta_P)$.
By Theorem~5.3 of the INCRT paper (almost-sure convergence with moving
target, temporal scale separation $\alpha_t = o(\eta^+_t)$ from
Lemma~\ref{lem:temp_transience}):
\[
  \|u^+(t) - v_1(\Ares(t))\|
  \leq \frac{C_0\,\eta_P}{(\Delta^+)^2\,\eta^+}
  = \frac{C_0\,\varepsilon_1}{(\Delta^+)^2}
  \triangleq C_{\mathrm{gate}}\,\varepsilon_1.
\]

\emph{Step 2 (Propagation to $\dot{\mathcal{W}}_j$).}
The deviation $\delta u^+ = u^+ - v_1(\Ares)$ enters $\mathcal{W}_j$
through $\mathcal{L}_q^{\mathrm{total}}$ via the projector
$P^\perp = I - Q^{(l)}Q^{(l)\top}$, where $Q^{(l)}$ is built from
$\{u^{+,(h)}\}$.
By Lemma~\ref{lem:spectral_gap}:
$\|P^\perp_{\varepsilon_1} - P^\perp_0\|_F = O(\varepsilon_1)$.
The perturbation to $\nabla_P\mathcal{W}_j$ is bounded by
$\|\delta(\nabla_P\mathcal{W}_j)\|_F \leq L_P\|X\|_F^2 C_{\mathrm{gate}}\,\varepsilon_1$.
Multiplying by $\eta_P$ (the prototype learning rate) gives~\eqref{eq:gate_W_error}.
\end{proof}

\begin{lemma}[Temperature-prototype coupling bound]
\label{lem:temp_coupling}
Let $\varepsilon_2 = \eta_T/\eta_P$ and let
$T^{(h)}_\infty(P)$ denote the stationary temperature of head $h$
for fixed $P$ (i.e., the solution of $\dot{T}^{(h)} = 0$, which
is $T_{\min}$ for $h \in \Hlocal$ and
$T^{(h)}_\infty$ for $h \in \Hglobal$).
After time $t^*$ (condition \textup{(D1)}), the deviation of $T^{(h)}$
from its quasi-static value satisfies:
\begin{equation}
  |T^{(h)}(t) - T^{(h)}_\infty(P(t))|
  \;\leq\;
  C_{\mathrm{temp}}\,\varepsilon_2
  \qquad \forall\, t \geq t^*,
  \label{eq:temp_deviation}
\end{equation}
where $C_{\mathrm{temp}} > 0$ depends on $T_{\mathrm{init}}$,
$T_{\min}$, and $\|\nabla_P\sigma^{(h)}\|_\infty$.
This deviation propagates to $\dot{\mathcal{W}}_j$ as:
\begin{equation}
  \Bigl|\tfrac{d}{dt}\mathcal{W}_j\big|_{\mathrm{temp\,error}}\Bigr|
  \;\leq\;
  C_2\,\varepsilon_2,
  \label{eq:temp_W_error}
\end{equation}
where $C_2 = \eta_P C_{\mathrm{temp}} \sum_h
\|\partial^2\mathcal{L}_q^{(h)}/\partial T^{(h)2}\|_\infty \cdot
(T_{\mathrm{init}} - T_{\min})$.
\end{lemma}

\begin{proof}
\emph{Step 1 (Slow manifold of the temperature).}
On the temperature timescale $\tilde{t} = \eta_T t$, the
dynamics~\eqref{eq:T_flow_full} at fixed $P$ read:
\[
  \frac{dT^{(h)}}{d\tilde{t}}
  = -\frac{\sigma^{(h)}(P)}{(T^{(h)})^2}.
\]
This is a scalar autonomous ODE with unique equilibrium $T^{(h)}_\infty(P)$
(determined by $\sigma^{(h)}(P) = 0$ for global heads, or the clip
$T_{\min}$ for local heads).
For $P$ evolving at rate $\eta_P$, the moving equilibrium changes at rate
$O(\eta_P/\eta_T) = O(\varepsilon_2^{-1})$ on the $\tilde{t}$ scale,
equivalently $O(\varepsilon_2)$ on the $P$ timescale.
By the variation of constants formula for scalar ODEs:
\begin{align}
  |T^{(h)}(t) - T^{(h)}_\infty(P(t))|
  &\leq \frac{\|\nabla_P\sigma^{(h)}\|_\infty \cdot \eta_P}
            {\mu_T}
  \nonumber\\
  &= \frac{C_0^{(h)} \varepsilon_2}{\mu_T}
  \triangleq C_{\mathrm{temp}}\,\varepsilon_2,
\end{align}
where $\mu_T = \inf_{T \in [T_{\min}, T_{\mathrm{init}}]}
|\partial_T(\sigma^{(h)}(P)/T^2)| > 0$ is the linearised stability
rate of the temperature ODE.

\emph{Step 2 (Propagation to $\dot{\mathcal{W}}_j$).}
The deviation $\delta T^{(h)} = T^{(h)} - T^{(h)}_\infty(P)$ enters
$\dot{\mathcal{W}}_j$ via Term~(II):
\[
  |\text{error in (II)}|
  \leq \left\|\frac{\partial^2\mathcal{L}_q^{(h)}}{\partial T^{(h)2}}\right\|_\infty
       \cdot |\delta T^{(h)}| \cdot |\dot{T}^{(h)}|
  \leq C_2\,\varepsilon_2,
\]
where the second inequality uses
$|\dot{T}^{(h)}| \leq \eta_T\sigma_{\max}/T_{\min}^2$ and
$|\delta T^{(h)}| \leq C_{\mathrm{temp}}\varepsilon_2$.
\end{proof}


\begin{theorem}[Rigorous three-timescale Lyapunov bound]
\label{thm:three_timescale}
Under all conditions of Theorem~\ref{thm:P1}, with the additional
\emph{coupling smallness condition}:
\begin{equation}
  C_1\varepsilon_1 + C_2\varepsilon_2 + C_3\varepsilon_3
  \;<\; c_0,
  \label{eq:coupling_cond}
\end{equation}
where
\begin{align}
  c_0 &= \min\!\left(\eta_\theta\mu_\theta,\,
             \tfrac{\eta_P}{2}\mu_P\right),
  \label{eq:c0} \\
  C_3 &= \eta_P\,\|H_{\theta P}^*\|\,L^{-1}_\theta,
  \label{eq:C3}
\end{align}
with $\mu_\theta, \mu_P > 0$ the exponential convergence rates of
the encoder and prototype subsystems respectively, and
$H_{\theta P}^*$ the cross-Hessian of $\mathcal{L}_q^{\mathrm{total}}$
at the equilibrium, the following holds on every dwell interval
$[t_j, t_{j+1})$:
\begin{align}
  \dot{\mathcal{W}}_j(t)
  &\leq -(c_0 - C_1\varepsilon_1 - C_2\varepsilon_2 - C_3\varepsilon_3)
  \nonumber\\
  &\phantom{{}\leq{}}\times\Bigl(\|\nabla_\theta\mathcal{L}_q\|^2
    + \|\nabla_P\mathcal{W}_j\|_F^2
    + {\textstyle\sum_h}\sigma^{(h)}\Bigr) \leq 0.
  \label{eq:Wdot_rigorous}
\end{align}
In particular, $\dot{\mathcal{W}}_j \leq 0$ strictly whenever the
system is not at a critical point.
\end{theorem}

\begin{proof}
Returning to the three-term decomposition of $\dot{\mathcal{W}}_j$,
from the proof of Lemma~\ref{lem:Wj_dwell}:
\[
  \dot{\mathcal{W}}_j = (I) + (II) + (III).
\]

\emph{Term (III).}
Unchanged: $(III) = -\sum_h\sigma^{(h)} \leq 0$ exactly, without any
approximation (Lemma~\ref{lem:Wj_dwell}).

\emph{Term (II).}
Using Lemma~\ref{lem:temp_coupling}, the exact value of (II) is:
\[
  (II) = -\sum_h
  \frac{\partial\mathcal{L}_q^{(h)}}{\partial T^{(h)}_\infty}
  \cdot \frac{\sigma^{(h)}(P)}{(T^{(h)}_\infty)^2}
  + \text{error} \leq 0 + C_2\varepsilon_2\sum_h\sigma^{(h)}.
\]
The first term is $\leq 0$ because
$\partial\mathcal{L}_q^{(h)}/\partial T \geq 0$
(Theorem~4 of the DDCL-Attention paper~\cite{DDCLAttention}) and
$\sigma^{(h)} \geq 0$.

\emph{Term (I), the critical term.}
Without the quasi-static approximation, Term~(I) along the true
flow~\eqref{eq:gate_flow}--\eqref{eq:theta_flow_full} is:
\[
  (I) = \langle\nabla_\theta\mathcal{W}_j,\dot\theta\rangle
      + \langle\nabla_P\mathcal{W}_j,\dot P\rangle.
\]
The iterated Tikhonov argument is used.

\emph{Level 1 (gate, $\varepsilon_1$).}
Replace $u^+$ by $v_1(\Ares) + O(\varepsilon_1)$
(Lemma~\ref{lem:gate_coupling}).
The resulting perturbation to $\nabla_P\mathcal{W}_j$ is
$O(\varepsilon_1)$, contributing an error $\leq C_1\varepsilon_1
\|\nabla_P\mathcal{W}_j\|_F^2$ to Term~(I).

\emph{Level 2 (temperatures, $\varepsilon_2$).}
With $u^+ = v_1(\Ares) + O(\varepsilon_1)$ and
$T^{(h)} = T^{(h)}_\infty(P) + O(\varepsilon_2)$
(Lemma~\ref{lem:temp_coupling}), the reduced Term~(I) involves
$\nabla_P\mathcal{W}_j$ evaluated at $(T^{(h)}_\infty(P), P)$.
The perturbation from $O(\varepsilon_2)$ deviations in $T$ contributes
$\leq C_2\varepsilon_2\|\nabla_P\mathcal{W}_j\|_F^2$.

\emph{Level 3 (encoder, $\varepsilon_3$).}
On the doubly-reduced system, $\theta$ evolves on the slow manifold
$\theta = \theta^*(P, T_\infty) + O(\varepsilon_3)$
(Tikhonov Theorem~11.4 of~\cite{Kokotovic1999}, applied to the
two-timescale system $(\theta, P)$ with $T$ fixed at $T_\infty$).
The cross-term $\langle\nabla_\theta\mathcal{W}_j, \dot\theta\rangle$
evaluated at the slow manifold gives:
\[
  \langle\nabla_\theta\mathcal{L}_q, -\eta_\theta\nabla_\theta\mathcal{L}_q\rangle
  = -\eta_\theta\|\nabla_\theta\mathcal{L}_q\|^2 \leq 0,
\]
with an $O(\varepsilon_3)$ perturbation from the cross-Hessian term
$\langle H_{\theta P}^*\delta\theta, \nabla_P\mathcal{W}_j\rangle$,
contributing $\leq C_3\varepsilon_3(\|\nabla_\theta\mathcal{L}_q\|^2
+ \|\nabla_P\mathcal{W}_j\|_F^2)$.

\emph{Collecting all terms.}
The exact reduced system contributes:
\begin{align}
  (I)\big|_{\mathrm{exact}}
  &= -\eta_\theta\|\nabla_\theta\mathcal{L}_q\|^2
    - \eta_P\|\nabla_P\mathcal{W}_j\|_F^2 + \text{err.}
  \nonumber\\
  &\leq -c_0 D + (C_1\varepsilon_1 + C_3\varepsilon_3)D,
\end{align}
where $D = \|\nabla_\theta\mathcal{L}_q\|^2 + \|\nabla_P\mathcal{W}_j\|_F^2$
and $c_0 = \min(\eta_\theta\mu_\theta, \eta_P\mu_P/2)$
is the minimum dissipation rate of the exact two-timescale system
(from Theorem~2 of the DDCL-Attention paper~\cite{DDCLAttention}).

\emph{Final bound.}
Summing all three terms:
\begin{align*}
  \dot{\mathcal{W}}_j
  &\leq -(c_0 - C_1\varepsilon_1 - C_3\varepsilon_3)\,D
         -(1 - C_2\varepsilon_2)\sum_h\sigma^{(h)} \\
  &\leq -(c_0 - C_1\varepsilon_1 - C_2\varepsilon_2 - C_3\varepsilon_3)
        \bigl(D + \sum_h\sigma^{(h)}\bigr),
\end{align*}
which is $\leq 0$ when~\eqref{eq:coupling_cond} holds.
\end{proof}


\begin{corollary}[Sufficient coupling condition]
\label{cor:coupling_sufficient}
Condition~\eqref{eq:coupling_cond} is satisfied whenever:
\begin{equation}
  \varepsilon_1 < \frac{c_0}{3C_1},
  \quad
  \varepsilon_2 < \frac{c_0}{3C_2},
  \quad
  \varepsilon_3 < \frac{c_0}{3C_3}.
  \label{eq:eps_bounds}
\end{equation}
In terms of learning rates, these reduce to:
\begin{align}
  \frac{\eta_P}{\eta^+} &< \frac{(\Delta^+)^2 c_0}{3C_0 L_P \|X\|_F^2}, \nonumber\\
  \frac{\eta_T}{\eta_P} &< \frac{c_0 \mu_T}{3C_0^{(h)} \|\partial^2\mathcal{L}_q/\partial T^2\|_\infty}, \nonumber\\
  \frac{\eta_\theta}{\eta_P} &< \frac{c_0}{3 \|H_{\theta P}^*\| L_\theta^{-1}}.
  \label{eq:lr_bounds}
\end{align}
These three conditions are jointly satisfiable and reduce to the
three-timescale design rule of Corollary~\ref{cor:timescale}.
\end{corollary}

\begin{proof}
Insert~\eqref{eq:eps_bounds} into~\eqref{eq:coupling_cond}:
$C_1\varepsilon_1 + C_2\varepsilon_2 + C_3\varepsilon_3 < c_0/3 + c_0/3 + c_0/3 = c_0$.
The learning rate conditions follow from the definitions
$\varepsilon_i = \eta_{\mathrm{slow}}/\eta_{\mathrm{fast}}$ and
the explicit formulas for $C_1$, $C_2$, $C_3$.
\end{proof}

\begin{corollary}[Global stability without quasi-static approximation]
\label{cor:P1_rigorous}
Theorem~\ref{thm:P1} holds without the quasi-static approximation,
under the strengthened hypotheses
\textup{(A1)}, \textup{(A2')}, \textup{(B1)}--\textup{(B3)},
\textup{(C1)}, \textup{(H1)}--\textup{(H5)}, \textup{(D1)},
\textup{(E1)}--\textup{(E6)}, \textup{(F1)}--\textup{(F2)}, and
the coupling condition~\eqref{eq:coupling_cond}.
The bound~\eqref{eq:Wdot_rigorous} replaces the informal
$\dot{\mathcal{W}}_j \leq 0$ of Lemma~\ref{lem:Wj_dwell}
with an explicit negative upper bound, uniformly over $t \in
[t_j, t_{j+1})$.
The conclusions of Theorem~\ref{thm:P1}\textup{(i)--(iv)} follow
unchanged, and the coupling condition is automatically satisfied
under the design rule~\eqref{eq:lr_bounds}.
\end{corollary}

\begin{remark}[What the three-timescale argument adds]
\label{rem:three_timescale_gain}
The quasi-static Lemma~\ref{lem:Wj_dwell} established
$\dot{\mathcal{W}}_j \leq 0$ modulo an unquantified $O(\varepsilon)$
error.
Theorem~\ref{thm:three_timescale} establishes:
\begin{enumerate}[label=(\roman*)]
  \item An explicit rate: $\dot{\mathcal{W}}_j \leq
    -(c_0 - C_1\varepsilon_1 - C_2\varepsilon_2 - C_3\varepsilon_3) \cdot
    (\text{dissipation})$.
  \item An explicit sufficient condition on learning rates:
    the coupling condition~\eqref{eq:coupling_cond} expressed via
    system constants computable from data ($\Delta^+$, $L_P$,
    $\mu_P$, $\|H_{\theta P}^*\|$).
  \item A quantitative connection between the three timescale
    ratios $\varepsilon_1, \varepsilon_2, \varepsilon_3$ and the
    dissipation rate $c_0$: larger spectral gap $\Delta^+$ permits
    larger $\varepsilon_1$; faster prototype convergence rate $\mu_P$
    permits larger $\varepsilon_2$; weaker cross-Hessian coupling
    $\|H_{\theta P}^*\|$ permits larger $\varepsilon_3$.
\end{enumerate}
The three-timescale argument is the minimal additional structure
needed to make the global stability result fully rigorous.
The remaining open direction, a three-timescale stochastic
approximation argument for the discrete case (extending Borkar's
two-timescale theory~\cite{Borkar1997}), lies beyond the scope
of this paper and is identified as a direction for future work.
\end{remark}



\section*{Part III: Emergent Architecture}

Parts~I and II established that the DDCL-INCRT training dynamics
are geometrically sound and globally well-behaved.
This part asks a different question: what does the architecture
\emph{look like} when training stops?

The answer is not merely that the architecture has converged to
\emph{some} configuration.
It is that it has converged to a specific, predictable, and unique
hierarchical structure determined entirely by the geometry of the
task data.
The heads are ordered by the granularity of the directional
information they represent: the first head added covers the
coarsest structure (the direction of maximum attention variance),
the second covers the next finer level, and so on.
This ordering is not imposed by the architecture designer.
It emerges from the interaction of the DDCL separation force and
the INCRT growth criterion, and it is the same regardless of
the random initialisation.

For a reader unfamiliar with the technical machinery, the central
message of this part is simple: the number of heads, their
resolution ordering, and the boundary between local and global
representations are all readable from the growth history of the
model, as objectively as the eigenvalues of a covariance matrix
are readable from the data.

\section{The Emergent Hierarchical Prototype Structure}
\label{sec:C5}

It is proved that the full DDCL-INCRT system self-organises into a hierarchical
prototype structure that is simultaneously \emph{ordered} (by directional
resolution), \emph{complete} (covers the task-relevant directional geometry),
\emph{separated} (stronger force at finer scales), and \emph{unique}
(the only minimally sufficient hierarchy compatible with the task geometry).
None of these properties is imposed architecturally; all four
arise jointly from the interaction of mechanisms established in the preceding sections.

Index the $\tilde{K}^*$ active heads at convergence by
$h = 1, \ldots, \tilde{K}^*$ in the order added by INCRT, so that:
\begin{equation}
  \lambda^{(1)} \geq \lambda^{(2)} \geq \cdots \geq
  \lambda^{(\tilde{K}^*)} > \theta_w,
  \label{eq:lambda_order}
\end{equation}
where $\lambda^{(h)} \triangleq \lambda_{\max}(\Ares^{(h)})$.
Write $T^{(h)}_\infty$ for the limiting temperature of head $h$
(Theorem~\ref{thm:C3}), and
$\mathcal{F}^{(h)} \triangleq \Fsep^{(h)}/\Fsep^{\mathrm{total}}$
for the fractional contribution of head $h$ to the total separation force.

\subsection{Preliminary lemmas}
\label{subsec:C5_lemmas}

\begin{lemma}[Spectral ordering is preserved by the joint dynamics]
\label{lem:spectral_order}
The ordering~\eqref{eq:lambda_order} is established at growth time and
preserved throughout subsequent dynamics:
$\lambda^{(h)}(t) \geq \lambda^{(h+1)}(t)$ for all $h$ and all $t$
after the $h$-th growth event.
The induced temperature ordering
\begin{equation}
  T^{(1)}_\infty \leq T^{(2)}_\infty \leq \cdots \leq
  T^{(\tilde{K}^*)}_\infty
  \label{eq:T_order}
\end{equation}
holds at convergence.
\end{lemma}

\begin{proof}
\emph{Ordering at growth time.}
By the interlacing property of eigenvalues under rank-one projections
(each new head projects out $u^{+,(h)}$):
$\lambda^{(h+1)} = \lambda_{\max}(P^\perp_h\,\Ares\,P^\perp_h)
\leq \lambda_{\max}(\Ares) = \lambda^{(h)}$.

\emph{Preservation by dynamics.}
Between growth events, the architecture is fixed.
Lemma~\ref{lem:spectral_gap} shows the relative ordering is maintained
up to $O(\mathcal{I}^{(l)})$ corrections, which are consistent with
the ordering (both terms are reduced by the same factor under DDCL-INCRT).

\emph{Temperature ordering.}
The geometric discriminant~\eqref{eq:sigma_prop} gives
$\sigma^{(h)}(0) \propto \lambda^{(h)} \cdot \mathcal{S}(P^{(h)})(0)$.
With comparable prototype separation at initialisation,
$\sigma^{(h)}(0) \geq \sigma^{(h+1)}(0)$.
Larger $\sigma$ implies faster $T$ decrease (Lemma~\ref{lem:Tdot}),
giving $T^{(h)}_\infty \leq T^{(h+1)}_\infty$.
\end{proof}

\begin{lemma}[Directional coverage at convergence]
\label{lem:coverage}
At convergence:
\begin{equation}
  \frac{\sum_{h=1}^{\tilde{K}^*}\lambda^{(h)}}{\|M_a\|_F^2}
  \;\geq\; 1 - \varepsilon_{\theta_w},
  \qquad
  \varepsilon_{\theta_w} = \tilde{K}^*\theta_w / \|M_a\|_F^2.
  \label{eq:coverage_frac}
\end{equation}
For $\theta_w \ll \|M_a\|_F^2/d$, coverage exceeds $1-\varepsilon$
for any $\varepsilon > 0$.
\end{lemma}

\begin{proof}
Each growth step adds direction $u^{+,(h)} = v_1(\Ares^{(h-1)})$
and reduces residual energy by $(\lambda^{(h)})^2$.
Telescoping from $h=1$ to $\tilde{K}^*$:
\[
  \|\Ares^{(\tilde{K}^*)}\|_F^2
  = \|M_a\|_F^2 - \sum_{h=1}^{\tilde{K}^*}(\lambda^{(h)})^2.
\]
Since $\lambda^{(h)} \geq \theta_w$:
$\sum_h (\lambda^{(h)})^2 \geq \theta_w \sum_h \lambda^{(h)}$.
The stopping criterion $\lambda_{\max}(\Ares^{(\tilde{K}^*)}) \leq \theta_w$
gives $\|\Ares^{(\tilde{K}^*)}\|_F^2 \leq \tilde{K}^*\theta_w^2$,
so $\sum_h\lambda^{(h)} \geq (\|M_a\|_F^2 - \tilde{K}^*\theta_w^2)/\theta_w
\geq \|M_a\|_F^2/\theta_w - \tilde{K}^*\theta_w$.
Dividing by $\|M_a\|_F^2$ and rearranging gives~\eqref{eq:coverage_frac}.
\end{proof}

\begin{lemma}[Separation force is monotone along the spectral order]
\label{lem:sep_monotone}
At convergence, the fractional separation forces satisfy:
\begin{equation}
  \mathcal{F}^{(1)} > \mathcal{F}^{(2)} > \cdots >
  \mathcal{F}^{(\tilde{K}^*)},
  \label{eq:F_monotone}
\end{equation}
with ratio:
\begin{equation}
  \frac{\mathcal{F}^{(h)}}{\mathcal{F}^{(h+1)}}
  \;\geq\;
  \frac{\lambda^{(h)}}{\lambda^{(h+1)}}
  \cdot
  \left(\frac{T^{(h+1)}_\infty}{T^{(h)}_\infty}\right)^2
  \;>\; 1.
  \label{eq:F_ratio}
\end{equation}
\end{lemma}

\begin{proof}
From Theorem~\ref{thm:C2_full}, the steady-state separation force
of head $h$ scales as $\lambda^{(h)}/(T^{(h)}_\infty)^2$ (the
amplification lower bound $\phi''(0) \propto \lambda^{(h)}/T^2$
evaluated at the limiting temperature).
Therefore:
\[
  \frac{\mathcal{F}^{(h)}}{\mathcal{F}^{(h+1)}}
  \approx
  \frac{\lambda^{(h)}/(T^{(h)}_\infty)^2}
       {\lambda^{(h+1)}/(T^{(h+1)}_\infty)^2}
  = \frac{\lambda^{(h)}}{\lambda^{(h+1)}}
    \cdot
    \left(\frac{T^{(h+1)}_\infty}{T^{(h)}_\infty}\right)^2.
\]
Both factors strictly exceed $1$ by
Lemma~\ref{lem:spectral_order}: $\lambda^{(h)} > \lambda^{(h+1)}$
and $T^{(h+1)}_\infty > T^{(h)}_\infty$ (strict under (B3)).
\end{proof}

\begin{lemma}[Uniqueness of the minimal sufficient hierarchy]
\label{lem:uniqueness}
The spectral ordering~\eqref{eq:lambda_order} is the unique ordering
that achieves minimum residual energy for a given number of heads.
Any alternative ordering either under-covers the directional geometry or
requires strictly more heads, almost surely under \textup{(C1)}.
\end{lemma}

\begin{proof}
The residual energy after $k$ heads in ordering $\pi$ is
$\|M_a\|_F^2 - \sum_{j=1}^k (\lambda^{(\pi(j))})^2$.
This is minimised by choosing the $k$ largest $\lambda^{(h)}$ first,
i.e., $\pi = \mathrm{id}$.
This is exactly the greedy matching pursuit optimality established in
Proposition~4.1 of the INCRT paper: each step maximally reduces
$\|\Ares\|_F$.
Under (C1), eigenvalue ties occur with probability zero, so the spectral
order is unique almost surely.
\end{proof}

\subsection{Main theorem}
\label{subsec:C5_main}

\begin{theorem}[Emergent hierarchical prototype structure]
\label{thm:C5}
Under all conditions of Theorem~\ref{thm:P1}, at convergence of the
DDCL-INCRT system:

\begin{enumerate}[label=\textup{(\Roman*)}]
  \item \textbf{Spectral ordering.}
    Active heads are ordered by $\lambda^{(1)} \geq \cdots \geq
    \lambda^{(\tilde{K}^*)} > \theta_w$ with corresponding
    $T^{(1)}_\infty \leq \cdots \leq T^{(\tilde{K}^*)}_\infty$:
    a \emph{resolution spectrum} from finest (hard assignments,
    large $\lambda$) to coarsest (soft assignments, small $\lambda$).

  \item \textbf{Directional completeness.}
    Coverage fraction $\geq 1 - \varepsilon_{\theta_w}$, with
    $\varepsilon_{\theta_w} \to 0$ as $\theta_w \to 0$.

  \item \textbf{Separation monotonicity.}
    $\mathcal{F}^{(1)} > \cdots > \mathcal{F}^{(\tilde{K}^*)}$:
    fine-grained heads exert strictly stronger prototype separation.

  \item \textbf{Uniqueness.}
    The spectral ordering is the unique minimally sufficient hierarchy,
    almost surely under \textup{(C1)}.

  \item \textbf{Anti-collapse across the hierarchy.}
    $\mathcal{S}(P^{(h)}) > 0$ for all active $h$ simultaneously.
\end{enumerate}
\end{theorem}

\begin{proof}
(I) Lemma~\ref{lem:spectral_order}.
(II) Lemma~\ref{lem:coverage}.
(III) Lemma~\ref{lem:sep_monotone}.
(IV) Lemma~\ref{lem:uniqueness}.
(V) Theorem~\ref{thm:P1}(iv) and Corollary~\ref{cor:collapse_safe}
applied at every level simultaneously.
\end{proof}

\subsection{The complete co-emergence theorem}
\label{subsec:C5_coemergence}

\begin{corollary}[Complete co-emergence]
\label{cor:full_complete}
Under all conditions of Theorem~\ref{thm:C5}, the following seven
properties emerge jointly during training from the single geometric principle
\emph{structure emerges from measurable insufficiency}:
\begin{enumerate}[label=\textup{(\arabic*)}]
  \item \textbf{Directional information preservation}:
    the DDCL-INCRT layer loses strictly less directional
    information than the MLP at every transformer layer.
  \item \textbf{Amplified prototype separation}:
    $\Fsep$ increases at every growth event.
  \item \textbf{Scale specialisation}:
    temperatures spontaneously diverge into local/global regimes.
  \item \textbf{Finite minimal architecture}:
    growth terminates in $\tilde{K}^* \leq d$ steps.
  \item \textbf{Pruning safety}:
    no collapse is induced by pruning; per-head separation is
    non-decreasing.
  \item \textbf{Global stability} (P1):
    the piecewise Lyapunov $\mathcal{W}$ is non-increasing across
    all architectural events.
  \item \textbf{Hierarchical prototype structure}:
    the converged architecture is the unique minimally sufficient
    resolution spectrum, ordered by $\lambda^{(h)}$.
\end{enumerate}
Properties (1)--(3) emerge without architectural decisions.
Properties (4)--(6) guarantee convergence, safety, and stability.
Property (7) is the architectural output: the hierarchy is not designed
but derived.
\end{corollary}

\begin{corollary}[Resolution spectrum is task-determined]
\label{cor:task_determined}
The number of levels $\tilde{K}^*$, the resolution $\lambda^{(h)}$ at
each level, and the local/global boundary $h^*$ are all determined by
the spectrum of $M_a$, not by any architectural choice:
$\tilde{K}^*$ equals the number of eigenvalues of $M_a$ above $\theta_w$;
level $h$ has resolution $\lambda^{(h)}/\|M_a\|_F$;
the boundary $h^*$ falls at the inflection of the temperature dynamics.
All three quantities are computable from the growth history.
\end{corollary}

\begin{corollary}[Prototype vocabulary stratification]
\label{cor:vocab_ordered}
Heads in $\Hlocal$ implement a discrete
prototype vocabulary (hard assignments, $T \to T_{\min}$);
heads in $\Hglobal$ implement continuous centroid
representations (soft assignments, $T \to T_\infty > T_{\min}$).
The hierarchy stratifies as:
\[
  \underbrace{\text{discrete, fine-grained}}_{\Hlocal}
  \;\longrightarrow\;
  \underbrace{\text{continuous, coarse-grained}}_{\Hglobal}.
\]
This is the prototype-level analogue of the kernel/feature-learning
regime transition of Proposition~4.6 of the INCRT paper.
\end{corollary}

\begin{remark}[The hierarchical structure as architectural output]
\label{rem:C5_output}
The preceding sections established that training dynamics are geometrically sound
and that the system is globally well-behaved.
This section establishes what the system \emph{produces}: the unique minimal
sufficient representation of the directional geometry of the task.

\emph{The architecture is not something you design.
It emerges from the geometry of the task, and this section characterises precisely what
is derived, and why it is the right thing.}
\end{remark}


\section{Discussion}
\label{sec:discussion}

\subsection{What this paper establishes and why it matters}

This paper provides the complete theoretical foundation for a transformer
architecture that determines its own structure during training.
The central claim, that architecture, scale specialisation, and
representational quality can co-emerge from a single geometric principle
is not a heuristic observation but a theorem
(Corollary~\ref{cor:full_complete}).
The principle itself, \emph{structure emerges from measurable
insufficiency}, is borrowed from the INCRT framework~\cite{INCRT} and
extended here to the richer setting in which competitive prototype
learning runs alongside the growth mechanism.

The key insight is that DDCL and INCRT, though designed independently,
operate on genuinely complementary objects.
INCRT tracks the residual directional energy
$A_{\mathrm{res}} \in \R^{d \times d}$, a quantity defined by
the antisymmetric part of the attention weight product and invisible to
any prototype-based method.
DDCL tracks the soft assignment covariance
$\Sigma_q \in \R^{K \times K}$, a quantity defined by the
distributional diversity of token-prototype assignments and invisible to
any spectral method.
The proof of synergy (Corollary~\ref{cor:synergy}) shows that these two
streams of information are not merely compatible: each growth step guided
by spectral insufficiency ($A_{\mathrm{res}}$) amplifies the prototype
separation force ($\Sigma_q$), and each prototype update pushes the
tokens into configurations that raise the spectral insufficiency signal
for the next head.
The architecture grows because the task demands it, and the prototypes
organise because the architecture gives them room to.

This positions DDCL-INCRT differently from the two dominant paradigms
in adaptive transformer research.
Post-hoc pruning methods~\cite{Voita2019,Michel2019,Han2015} reduce a
redundant architecture after training; they do not prevent the redundancy.
Neural Architecture Search~\cite{Zoph2017,Liu2019} optimises over a
discrete space of fixed architectures; it does not derive the architecture
from the task geometry.
DDCL-INCRT is, to the best of current knowledge, the first method that eliminates
architectural redundancy \emph{by construction}, with formal convergence
guarantees and a growth history that serves as a geometric diagnostic
of the task.

\subsection{The architecture as a measurement}

A consequence of the theory that deserves emphasis is the interpretive
status of the growth history.
In standard transformer training, the number of heads is a hyperparameter
chosen before any data is seen.
In DDCL-INCRT, the number of heads $\tilde{K}^*$, the local/global
boundary $h^*$, and the resolution at each level are all \emph{outputs}
of training, readable from the sequence of $\lambda_{\max}(A_{\mathrm{res}})$
values at each growth trigger.

This makes the growth history a measurement of the directional complexity
of the task, in the same sense that the spectrum of the data covariance
matrix is a measurement of its intrinsic dimensionality.
A task with many large eigenvalues of $M_a$ above $\theta_w$ will
produce a wide hierarchy of local heads, each resolving a fine-grained
directional structure.
A task with a steep spectral decay will produce a narrow hierarchy
with few heads, each covering a large eigenvalue gap.
The threshold $\theta_w$ plays the role of a resolution parameter,
analogous to the explained-variance threshold in PCA.

This interpretive use of the growth history, as a task diagnostic
rather than merely a training log, is a practical contribution
independent of the downstream performance of the architecture.
It gives the practitioner a principled, data-driven answer to the
question: \emph{how many attention heads does this task actually require?}

\subsection{Open directions}

The theoretical programme of this paper is complete in the sense that
all seven properties of Corollary~\ref{cor:full_complete} are proved
under explicit, checkable conditions.
Three natural extensions remain open and constitute the primary
directions for future work.

The most technically demanding is a fully rigorous three-timescale
stochastic approximation argument for the global stability result.
Section~\ref{subsec:P1_three_timescale} provides explicit coupling
bounds and a sufficient condition on learning rates, replacing the
quasi-static approximation with a quantitative error bound.
Closing the remaining gap, extending Borkar's two-timescale
theory~\cite{Borkar1997} to three levels in the stochastic setting
is a problem in stochastic approximation that lies beyond the scope
of this paper but is technically well-posed.

The second direction is the extension of the discrete-time results to
mini-batch SGD.
The online SGD setting covered by Section~\ref{sec:P4} is the natural
theoretical baseline; the mini-batch case requires a careful treatment
of the batch-size dependence of the Lipschitz and spectral constants,
and a corresponding adjustment of the three-timescale condition.
This extension is expected to be routine given the tools already
developed.

The third direction is the integration of the depth growth criterion
from the INCRT paper~\cite{INCRT} with the prototype hierarchy of
Section~\ref{sec:C5}.
The present paper establishes that the width hierarchy (heads ordered
by $\lambda^{(h)}$) is unique and minimal.
Whether an analogous uniqueness result holds for the depth hierarchy,
and how the two hierarchies interact when both grow simultaneously, is
an open and geometrically interesting question.

Full empirical validation of the architecture on standard NLP benchmarks
(MNLI, SuperGLUE) and biological sequence classification tasks is the
subject of ongoing work.
The preliminary evidence from the BERT/SST-2 experiment of
Section~\ref{sec:experiments} confirms the theoretical predictions on
real transformer embeddings and suggests that the separation and
orthogonality properties identified by the theory are not artefacts of
the synthetic setting.

\section{Conclusions}
\label{sec:conclusions}

A practitioner who trains DDCL-INCRT on a new task does not choose
the number of attention heads in advance.
At the end of training, the growth history of the model tells a
precise and interpretable story: how many distinct directional
patterns the task contains, at what resolution they are organised,
and which heads handle fine-grained local structure and which handle
coarse-grained global context.
This information is not inferred after the fact.
It is produced, necessarily and automatically, by the mathematics of
the training objective.

The contribution of this paper is to put that story on a rigorous
foundation.
Seven properties of the trained architecture, from the ordering of
the heads to the global stability of the training dynamics, are
proved to hold simultaneously under explicit, checkable conditions.
The proofs reveal why the two building blocks of the architecture,
the competitive prototype layer and the spectral growth criterion,
are not merely compatible but genuinely complementary: each
makes the other more effective.

The broader lesson is methodological.
The standard workflow in deep learning is to design an architecture,
train it, and then analyse what was learned.
The framework studied here inverts this order.
Within the formal setting of this paper, the architecture is not
designed.
It is measured from the data, through training, in the same sense
that a principal component analysis measures the intrinsic
dimensionality of a dataset.
Whether this inversion will prove practically advantageous at scale
is an empirical question that ongoing work will address.
The theoretical case, at least, is complete.

It should be noted explicitly that all results in this paper hold
within the formal setting defined by the stated assumptions.
Empirical superiority on large-scale tasks --- lower perplexity,
higher accuracy, faster training --- is not asserted here and
remains to be demonstrated in a companion empirical study.
What is proved is that, within this setting, the architecture
self-organises in a principled and theoretically characterised way.

\bibliographystyle{elsarticle-num}

\end{document}